%% file: main.tex
\documentclass{article}

\usepackage{microtype}
\usepackage{graphicx}
\usepackage{subfigure}
\usepackage{booktabs} % for professional tables
\usepackage{enumerate}

\usepackage{hyperref}

\usepackage[accepted]{icml2024}

\usepackage{amsmath}
\usepackage{amssymb}
\usepackage{mathtools}
\usepackage{amsthm}

\usepackage{multirow}

\usepackage[capitalize,noabbrev]{cleveref}

\theoremstyle{plain}
\newtheorem{theorem}{Theorem}[section]

\newtheorem{lemma}[theorem]{Lemma}

\theoremstyle{definition}

\theoremstyle{remark}

\usepackage[textsize=tiny]{todonotes}

\icmltitlerunning{} % {Submission and Formatting Instructions for ICML 2024

\include{math_commands}

\usepackage{lipsum}
\usepackage{algorithm}
\usepackage{algpseudocode}
\usepackage{thmtools} 
\usepackage{thm-restate}
\usepackage{comment}
\usepackage{wrapfig}
\usepackage[super]{nth}
\usepackage{xcolor,colortbl}
\newcommand\cincludegraphics[2][]{\raisebox{-0.4\height}{\includegraphics[#1]{#2}}}

\begin{document}

\twocolumn[
\icmltitle{Bespoke Non-Stationary Solvers \\ for Fast Sampling of Diffusion and Flow Models}



\icmlsetsymbol{equal}{*}

\begin{icmlauthorlist}
\icmlauthor{Neta Shaul}{wis}
\icmlauthor{Uriel Singer}{metag}
\icmlauthor{Ricky T. Q. Chen}{metaf}
\icmlauthor{Matthew Le}{metaf}
\icmlauthor{Ali Thabet}{metag}
\icmlauthor{Albert Pumarola}{metag}
\icmlauthor{Yaron Lipman}{metaf,wis}
\end{icmlauthorlist}

\icmlaffiliation{wis}{Weizmann Institute of Science}
\icmlaffiliation{metag}{GenAI, Meta}
\icmlaffiliation{metaf}{FAIR, Meta}

\icmlcorrespondingauthor{Neta Shaul}{Neta.Shaul@weizmann.ac.il}



\vskip 0.3in
]



\printAffiliationsAndNotice{} % otherwise use the standard text. \icmlEqualContribution

\begin{abstract}
 This paper introduces Bespoke Non-Stationary (BNS) Solvers, a solver distillation approach to improve sample efficiency of Diffusion and Flow models. BNS solvers are based on a family of non-stationary solvers that provably subsumes existing numerical ODE solvers and consequently demonstrate considerable improvement in sample approximation (PSNR) over these baselines. Compared to model distillation, BNS solvers benefit from a tiny parameter space ($<$200 parameters), fast optimization (two orders of magnitude faster), maintain diversity of samples, and in contrast to previous solver distillation approaches nearly close the gap from standard distillation methods such as Progressive Distillation in the low-medium NFE regime. For example, BNS solver achieves 45 PSNR / 1.76 FID using 16 NFE in class-conditional ImageNet-64. We experimented with BNS solvers for conditional image generation, text-to-image generation, and text-2-audio generation showing significant improvement in sample approximation (PSNR) in all. 
 
\end{abstract}

\input{tables/teaser}

\section{Introduction}
\label{Introduction}
Diffusion and flow-based methods are now established as a leading paradigm for generative models of high dimensional signals including images~\cite{rombach2021highresolution}, videos~\cite{singer2022makeavideo} audio~\cite{vyas2023audiobox}, 3D geometry~\cite{yariv2023mosaicsdf}, and physical structures such as molecules and proteins~\cite{hoogeboom2022equivariant}. While having an efficient training algorithms, sampling is a costly process that still requires tens to hundreds of sequential function evaluations to produce a sample.

\pagebreak
Sample efficiency of generative models is crucial to enable certain applications, \eg, ones that require interaction with a user, as well as to reduce carbon footprint and costs of these, now vastly popular models. An ongoing research effort targets reducing sampling complexity of diffusion/flow models, concentrating on three main venues: (i) \emph{dedicated solvers}: employing high-order numerical ODE solvers \cite{karras2022elucidating} and/or time and scale reparameterizations to further simplify sample trajectories \cite{zhang2022fast}; (ii) \emph{model distillation}: fine-tuning the model to approximate the original model's samples or the training data with less function evaluations \cite{salimans2022progressive, meng2023distillation,liu2022flow}. Recently also perception and GAN discriminator losses have been incorporated to improve perception quality~\cite{yin2023onestep}; (iii) \emph{solver distillation}: a rather recent and unexplored approach that limits distillation to \emph{optimizing a numerical solver} that effectively sample the original model (that is kept frozen). While model distillation is generally able to reduce the number of function evaluations (NFEs) for producing samples with high perceptual scores, it can reduce the diversity of the model and shift the generated distribution. Furthermore, it is still costly to train (\ie, conceptually continues training the original model), and mostly requires access to the original training data. In contrast, solver distillation enjoys a tiny parameters space (\ie, $<$200 of parameters), very fast optimization compared to model distillation (\ie, by two order of magnitude), and does not require access to training data~\cite{shaul2023bespoke}. 

The main goal of this paper is to introduce Bespoke Non-Stationary (BNS) solvers, a solver distillation approach that provably subsumes all previous dedicated and distillation solvers (that we are aware of). While BNS solver family enjoys higher expressive power, it still inherits other model distillation properties such as tiny parameter space and fast training. Its higher expressive power demonstrates a considerable improvement in approximating the original model's samples (PSNR) for lower NFEs, and is able to nearly close the gap with standard model distillation approaches such as Progressive Distillation~\cite{salimans2022progressive} in terms of perception quality (FID) for low-medium NFE range (\ie, 8-16). We have experimented with BNS solvers for conditional image generation, Text-to-Image (T2I) generation, and Text-to-Audio generation. In all cases BNS solvers considerably improved PSNR of generated samples. Figure \ref{fig:teaser} depicts BNS sampling from a large scale T2I model using 16 NFE producing consistent samples to the Ground Truth (GT) samples, while baselines fail to achieve this consistency. A secondary goal of this paper is to provide a full taxonomy of popular numerical solvers used to sample diffusion and flow models, as well as present them in a consistent way that highlights their relations. 

We summarize the paper's contributions:
\begin{enumerate}[(i)]
    \item Introduce BNS solvers; subsumes existing  solvers.
    \item A simple and effective BNS optimization algorithm.
    \item Significantly improving sample approximation (PSNR) over existing solvers, and reducing the gap in perception (FID) from model distillation techniques. 
    \item Provide a full taxonomy of numerical solver used for sampling diffusion and flow models. 
\end{enumerate}

\section{Preliminaries}\label{s:prelims}
\paragraph{Flow-based generative models.}
We let $x\in \Real^d$ represent a signal, \eg, an image in pixel or latent space. Deterministic sampling of a diffusion or flow model is done by solving an Ordinary Differential Equation (ODE), 
\begin{equation}\label{e:ode}
    \dot{x}(t) = u_t(x(t)),
\end{equation}
where $x(t)$ is called a \emph{sample trajectory} initialized with $x(0)=x_0$, where $x_0\sim p_0(x_0)$ is a sample form the source distribution $p_0$ usually representing noise, and the ODE is solved until time $t=1$. The Velocity Field (VF) $u:[0,1]\times\Real^d\too \Real^d$ is defined using the provided diffusion/flow model, and is detailed below for popular model parametrizations. Note that we use the convention of ODE going forward in time with $t=0$ corresponding to source/noise and $t=1$ to data. 

\paragraph{Diffusion and Flow-Matching models.} There are three common model parametrizations used in Diffusion/Flow-Matching: (i) \emph{(Diffusion)} $\eps$-prediction~\cite{ho2020denoising}, (ii) \emph{(Diffusion)} the $x$-prediction~\cite{salimans2022progressive}, and (iii) \emph{(Flow-Matching)} velocity field (VF) prediction \cite{lipman2022flow}. We use the common notation $f:[0,1]\times\Real^d\too\Real^d$ to denote all three. The model $f$ is commonly trained on a predefined time dependent probability density path $p_t$, 
\begin{equation}\label{e:pt}
    p_t(x) = \int p_t(x|x_1)q(x_1)dx_1,
\end{equation}
where $q(x_1)$ denotes the data distribution, and the \emph{conditional probability path} $p_t(x|x_1)$ is often chosen to be a \emph{Gaussian path}, that is defined by a Gaussian kernel,
\begin{equation}\label{e:def_gaussian_path}
    p_t(x|x_1) = \gN(x|\alpha_t x_1, \sigma_t^2I),
\end{equation}
and the pair of time dependent functions $\alpha,\sigma:[0,1]\too[0,1]$ are called a \emph{scheduler} and satisfy 
\begin{equation}
    \alpha_0=0=\sigma_1,\ \  \alpha_1=1,\ \ \sigma_0>0.
\end{equation}
%Furthermore, the log SNR function is a strictly monotonically increasing function defined via the scheduler by $\lambda_t=\log(\alpha_t/\sigma_t)$. 
%
\input{tables/velocity_field_coefficients}%
All schedulers discussed in this paper (and those practically used in the literature) have strictly monotonically increasing Signal-to-Noise (SnR) ratio, defined by $\mathrm{snr}(t)=\alpha_t/\sigma_t$. For Gaussian paths, the velocity field $u_t$ (used for sampling, \eqref{e:ode}) takes the form
\begin{equation}\label{e:u_gaussian_path}
    u_t(x) = \beta_t x + \gamma_t f_t(x),
\end{equation}
where the coefficients $\beta_t,\gamma_t$ are given in Table \ref{tab:velocity_field_coefficients}. 
\paragraph{ST transformations and post-training scheduler change.} A Scale-Time (ST) transformation~\cite{shaul2023bespoke} transforms sample trajectories $x(t)$ according to the formula
\begin{equation}\label{e:bar_x}
    \bar{x}(r) = s_r x(t_r),
\end{equation}
where $t:[0,1]\too[0,1]$ and $s:[0,1]\too\R_{>0}$ are a time and scale reparameterization functions satisfying $t_0=0$, $t_1=1$ and $s_0,s_1>0$. These conditions in particular imply that $\bar{x}(1)=s_1x(1)$, that is, we can recover the original sample $x(1)$ from the transformed path's sample $\bar{x}(1)$ via $x(1)=s_1^{-1}\bar{x}(1)$. Consequently, ST transformations can potentially simplify the sample trajectories for approximation while still allowing to recover the model's original samples. The transformed VF $\bar{u}$ that generates the ST-transformed paths $\bar{x}(r)$ is 
\begin{equation}\label{e:bar_u}
    \bar{u}_r(x) = \frac{\dot{s}_r}{s_r}x + \dot{t}_r s_r u_{t_r}\parr{\frac{x}{s_r}}.
\end{equation}
Particular instances of this formula are also derived and/or discussed in \cite{karras2022elucidating,zhang2022fast,pokle2023training,kingma2021variational}. For strictly monotone SnR, Scale-Time transformations $(s_r,t_r)$ are in a 1-1 correspondence with a scheduler change $(\alpha_t,\sigma_t)\too (\bar{\alpha}_t,\bar{\sigma}_t)$ in the Gaussian probability path (\eqref{e:pt}), and the conversion between the two can be done using the following formulas \cite{shaul2023bespoke}:
\begin{equation}\label{e:conversion}
    \left . \begin{aligned}
    \bar{\alpha}_r &= s_r \alpha_{t_r} \\\bar{\sigma}_r &= s_r \sigma_{t_r}
    \end{aligned} \right \} \iff
    \left \{\begin{aligned}
    t_r &= \mathrm{snr}^{-1}(\overline{\mathrm{snr}}(r)) \\ s_r &= {\bar{\sigma}_r}/{\sigma_{t_r}}
    \end{aligned}\right. .
\end{equation}
In particular, given a VF $u$ trained with a Gaussian path defined by a scheduler $(\alpha_t,\sigma_t)$, moving to a different scheduler post-training can be done by first computing the ST transformation $(s_r,t_r)$ from \eqref{e:conversion} and then using $\bar{u}$ in \eqref{e:bar_u} and sample with \eqref{e:ode}.

% \begin{restatable}{theorem}{equivalence}\label{thm:equivalence}
% Consider a Gaussian path $p_t$ defined by a scheduler $(\alpha_t,\sigma_t)$ with sample trajectories $x(t)$. Then applying an arbitrary ST transformation will result in a different Gaussian path $\bar{p}_t$ defined by some scheduler $(\bar{\alpha}_r,\bar{\sigma}_r)$ with trajectories $\bar{x}(r)$. Furthermore, every Gaussian path $\bar{p}_t$ can be reached with an ST transformation applied to $p_t$. 
% \end{restatable}

\section{Bespoke Non-Stationary Solvers}

In this section we introduce and analyze the main object of this paper: \emph{Bespoke Non-Stationary} (BNS) solvers. We start with introducing the \emph{Non-Stationary} (NS) solvers family followed by developing an algorithmic framework to search within this family a particular solver suitable to sample a provided pre-trained diffusion or flow model. We call such a solver BNS solver. We conclude this section with a theoretical analysis providing a complete taxonomy for popular ODE solvers used for diffusion/flows sampling and proving that NS solvers subsumes them all, see Figure \ref{fig:ven-diagram}. 

\begin{wrapfigure}[7]{r}{0.4\columnwidth} \vspace{-45pt}
  \begin{center}
    \includegraphics[width=0.4\columnwidth]{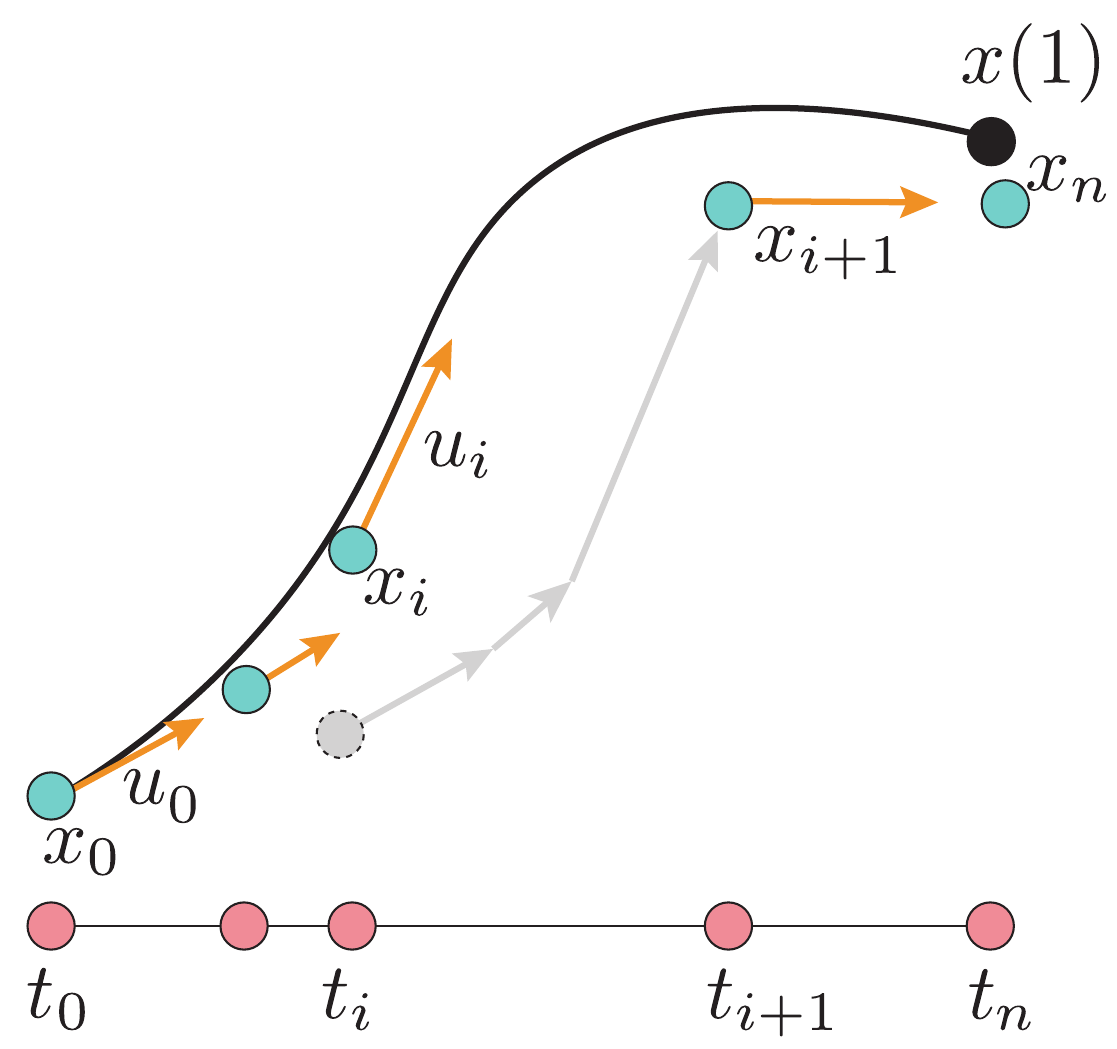}
  \end{center}\vspace{-10pt}
  \caption{Setup.}\label{fig:setup}
\end{wrapfigure}
\subsection{Non-Stationary Solvers}
In practice, \eqref{e:ode} is solved using a numerical ODE solver. We consider a broad family of ODE solvers - the \emph{Non-Stationary} (NS) Solvers. An $n$-step NS solver is defined by a pair: (i) a time-step discretization, and (ii) a set of $n$ update rules. The time-step discretization is a monotonically increasing sequence,
\begin{equation}\label{e:def_time_disc}
    T_n = \big(t_0, t_1, \ldots, t_{n-1}, t_n\big ), 
\end{equation}
always starts at $t_0=0$ and ends with $t_n=1$. The $i$-th update rule, where $i=0,\ldots,n-1$, has the form
\begin{equation}\label{e:def_nsms}
    x_{i+1} =  X_ic_i + U_id_i,
\end{equation}
where the matrix $X_i\in \R^{d \times (i+1)}$ stores all previous approximated points on the sample trajectory until and including time $t_i$, and the matrix $U_i \in \R^{d \times (i+1)}$ stores all velocity vectors evaluated at those previous samples, 
\begin{equation*}\label{e:def_x_matrix}
    X_i = \brac{\begin{array}{cccc}
        \vertbar & \vertbar & & \vertbar\\
        x_{0}&  x_{1} & \cdots & x_{i}\\
        \vertbar & \vertbar & & \vertbar
    \end{array}}, U_i = \begin{bmatrix}
    \vertbar & \vertbar & & \vertbar\\        
        u_0& u_1 & \cdots & u_i\\
        \vertbar & \vertbar & & \vertbar        
    \end{bmatrix},
\end{equation*}
where we denote $u_j=u_{t_j}(x_j)$, and the vectors $c_i, d_i \in \R^{i+1}$ are the parameters of the $i$-th step, see Figure \ref{fig:setup}. The $i$-th step outputs $x_{i+1}$ that approximates the GT sample at the same time, \ie, $x(t_{i+1})$. The NS solver can utilize an arbitrary linear combination of all previous points on the trajectory and their corresponding velocities.

\begin{figure}
    \centering
    \includegraphics[width=\columnwidth]{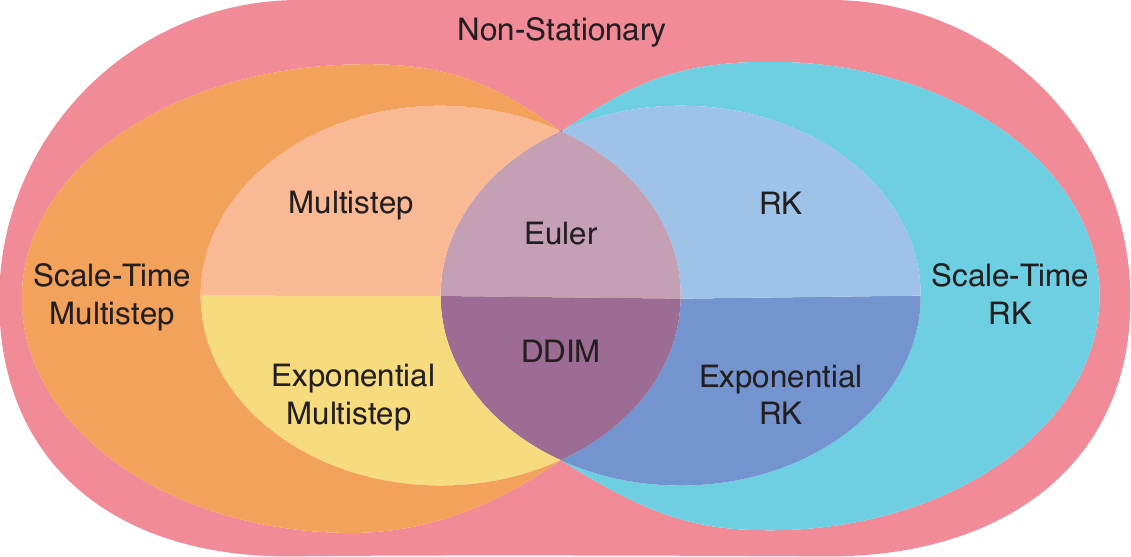}
    \caption{Taxonomy of ODE solvers used for sampling of diffusion/flow generative models.}
    \label{fig:ven-diagram}
\end{figure}

\subsection{Optimizing BNS Solvers}
Our goal is to find a member in the NS family of solvers that provides a good sampler for a \emph{specific} diffusion or flow model $u$. We call such a model-specific solver \emph{Bespoke Non-Stationary} (BNS) solver. In order to find an efficient BNS solver we require: (i) a parameterization $\theta\in\Real^p$ of NS solvers family; (ii) a cost function, $\gL(\theta)$, quantifying the effectiveness of different NS solver candidates in sampling $u$; and (iii) an initialization $\theta=\theta_0$ for optimizing the cost function. We detail these next.

\input{algorithms/nsms_sampling}

\paragraph{NS solvers parameterization.} The naive representation of an NS solver following the above introduction would be to collect the time discretization vector $T_n$ and all pairs $(c_i,d_i)$, $i=0,\ldots,n-1$,  defining the update rules in \eqref{e:def_nsms}. Although doable this would provide an \emph{over-parameterized} representation, meaning an NS solver can be represented in more than a single way. The following proposition provides a generically unique representation: \pagebreak
\begin{restatable}{proposition}{parameterization}\label{prop:ns_param}
    For every update rule $(c_i,d_i)\in \Real^{i+1}\times\Real^{i+1} $ of an NS solvers there a exist a pair $(a_i, b_i)\in \R\times\R^{i+1}$ so that the update rule can be equivalently written as 
    \begin{equation}\label{e:unique_nsms}
        x_{i+1} = x_0a_i + U_ib_i.
    \end{equation}
Furthermore, if the columns of $U_i$ are linearly independent then the pair $(a_i, b_i)$ is unique.
\end{restatable}
The proposition is proved using induction in Appendix \ref{a:bns_optim}. We note that \cite{Duan_2023} shows a similar result for diffusion $\eps$-prediction vectors and the special case of $X_i\in\Real^{d}$ (instead of the more general $X_i\in\Real^{d\times (i+1)}$).  

With \eqref{e:unique_nsms} as the new NS update rules the complete set of parameters $\theta\in \Real^p$ representing an NS solvers is
\begin{equation}
    \theta = \brac{ T_n , (a_0,b_0), \ldots, (a_{n-1},b_{n-1})},
\end{equation}
where the number of parameters is $p=n\parr{\frac{n+5}{2}}+1$, which is the dimension of $n$-steps NS solvers. Algorithm \ref{alg:nsms_sampling} shows how to generate a sample with an NS solver.

\paragraph{Cost function.}
To find an effective NS solver $\theta_*$ we consider a set of pairs $(x_0,x(1))$, where $x_0\sim p_0(x_0)$ are source samples, and $x(1)$ are high accuracy approximate solutions of \eqref{e:ode} with $x(0)=x_0$ as initial conditions. Then, we optimize the PSNR loss,
\begin{equation}\label{e:loss}
    \gL(\theta) = -\E_{(x_0,x(1))} \log \norm{x_n^\theta - x(1)}^2,
\end{equation}
where $x_n^\theta$ is the output of Algorithm \ref{alg:nsms_sampling} initialized with $x_0$, the velocity field $u$, and $\theta$; we denote $\norm{x}^2=\frac{1}{d}\sum_{i=1}^dx_i^2$.

\paragraph{Initialization and preconditioning.}
The last remaining part of our method is initialization $\theta=\theta_0$ and preconditioning, which are related. To have an effective optimization of the loss and reach a good solution we would like to start from an already reasonable solver. For that end we simply take $\theta_0$ to coincide with a generic ODE numerical solver such that Euler (RK 1$^{\text{st}}$ order) or Midpoint (RK 2$^{\text{nd}}$ order), see definition in Appendix \ref{a:generic_solvers}. This is always possible since, as we show in Section \ref{s:expressive_power}, all generic solvers are particular instances of NS solvers. However, just providing a good initialization is not always enough for successful optimization as bad conditioning can lead to either diverging solutions or excruciating slow convergence \cite{nocedal1999numerical}. We found that in some cases, especially when using high Classifier Free Guidance (CFG) scale \cite{ho2022classifier} preconditioning the velocity field $u$ by first changing its original scheduler improves convergence of $\theta$ and reaches better solutions in general. In particular we denote by $\sigma_0>0$ a \emph{preconditioning hyperparameter} and change the velocity field $u$ to $\bar{u}$ according to the scheduler 
\begin{equation}\label{e:initialization_scheduler}
    \bar{\sigma}_t = \sigma_0 \sigma_{t}, \quad\bar{\alpha}_t=\alpha_t,
\end{equation}
which corresponds to changing the source distribution to be proportional to $p_0(\frac{x}{\sigma_0})$, \ie, larger standard deviation. This is done using equations \ref{e:bar_u},\ref{e:conversion} as described in detail in Section \ref{s:prelims}. The BNS optimization is provided in Algorithm \ref{alg:nsms_training}.

\input{algorithms/nsms_training}

\input{tables/graphs_imagenet_main}

\subsection{Expressive power of Non-Stationary Solvers}
\label{s:expressive_power}
In this section we start by reviewing \emph{generic solvers}, move to \emph{dedicated solvers}, explained in a unified way with the aid of the ST transformation tool, and conclude with a full solver taxonomy theorem. In particular, this theorem shows that Non-Stationary solvers subsumes all other solvers. The solver taxonomy is illustrated in Figure \ref{fig:ven-diagram}. 

% of existing solvers used to sample diffusion and flow model; second, demonstrate previous diffusion-specific solvers can all be explained in a unified language using the ST transformations (Section \ref{s:prelims}); and finally, prove that Non-Stationary solvers subsumes all these solvers as particular instances. 

% We start with describing \emph{generic solvers}, next move to \emph{dedicated solvers}, and conclude with our taxonomy theorem. The full taxonomy is illustrated in Figure \ref{fig:ven-diagram}. 

% ODE sampling of diffusion/flow models have either: (i) utilized existing generic solvers, such as Runge-Kutta (RK) methods \cite{karras2022elucidating, lipman2022flow}; (ii) developed \emph{dedicated} solvers taking advantage of the specific structure of Gaussian paths and applied Exponential Integrators \cite{song2020denoising, karras2022elucidating, zhang2022fast, lu2023dpmsolver}; or (iii) trained a model-specific solver \cite{watson2021learning, shaul2023bespoke, Duan_2023}. We next provide a quick survey of the different approaches and discuss their relation and conclude by presenting our solver taxonomy theorem. 

\subsubsection{Generic solvers}\label{sss:generic_solvers}
Generic solvers build a series of approximations, $x_0,\ldots,x_i,x_{i+1},\ldots,x_n$, to the solution of the ODE in \eqref{e:ode} by iteratively applying an \emph{update rule}. The common update rules are derived from the following formula describing the solution to the ODE at time $t_{i+1}$ based on a known solution at time $t_i$, \ie,
\begin{equation}\label{e:exact_sol}
    x(t_{i+1}) = x(t_{i}) + \int_{t_{i}}^{t_{i+1}}u_t(x(t))dt.
\end{equation}
Generic solvers numerically approximate the integral by using, \eg, a polynomial approximation of $u_t(x(t))$ in the interval $[t_i, t_{i+1}]$, resulting in a \emph{stationary} (\ie, independent of $i$) update rule: \emph{Adam-Bashforth and Multistep methods} build the approximation based on \emph{past} times $t_{i-m+j}$, $j=1\ldots,m$, while \emph{Runge-Kutta methods} approximate it by \emph{future} times inside the interval $[t_i,t_{i+1}]$. Appendix \ref{a:generic_solvers} provides detailed formulas of both families, where a more elaborate exposition can be found in~\cite{iserles2009first}. 

\subsubsection{Dedicated solvers}
In this section we cover dedicated solvers developed specifically for diffusion and flow models, taking advantage of the particular structure of the Gaussian probability path (\eqref{e:pt}) and VF form (\eqref{e:u_gaussian_path}). Interestingly, all can be explained using scheduler change / ST transformations (detailed in Section \ref{s:prelims}) and application of a generic solvers.

% \yl{XXX}
% can used to describe simply many suggested method for developing dedicated solvers.
% \yl{EDM and exponential is the same transformation?}
% \ns{
% ST of exponential is
% \begin{equation}
%     s_r = \frac{1}{\psi_r},\quad t_r = r,
% \end{equation}
% where as the ST transformation of EDM is
% \begin{equation}
%     s_r = \frac{1}{\alpha_{t_r}}, \quad t_r = \text{snr}^{-1}\parr{\frac{1}{\sigma_{max}(1-r)}},
% \end{equation}
% where $\text{snr}(t)=\alpha_t/\sigma_t$. Note, in the case of $x$-prediction where $\psi_r=\sigma_r$ they get a strange probabilty path where the noise is constant $1$ and the signal is exploding...
% }

%------------------------------------------------------------------------------
%------------------------------------------------------------------------------
%------------------------------------------------------------------------------
\pagebreak
\textbf{EDM~\cite{karras2022elucidating}} change the original model's scheduler $(\alpha_t,\sigma_t)$ to 
\begin{equation}\label{e:edm_scheduler}
    \bar{\alpha}_r = 1,\quad \bar{\sigma}_r = \sigma_{\text{max}}(1-r),
\end{equation}
where $\sigma_{\text{max}}=80$. This scheduler transforms the original conditional paths to $p_t(x|x_1)=\gN(x|x_1,\sigma_{\max}(1-r))$ so that at time $r=0$, assuming $x_1$ has zero mean and std $\ll\sigma_{\max}=80$, the probability path approximates the Gaussian 
\begin{equation}\label{e:gaussian_max}
 p_0\approx \gN(0,\sigma_{\text{max}}^2I).   
\end{equation}
This scheduler is often called Variance Exploding (VE) due to the large noise std. Note that $\sigma_{\max}$ has to be sufficiently large for \eqref{e:gaussian_max} to hold. In contrast, our initialization utilizes a target scheduler that at time $r=0$ reaches arbitrary desired std (the hyperparameter $\sigma_0$) with no bias, \ie, $p_0=\gN(0,\sigma_0^2 I)$. In practice we find that setting $\sigma_0$ too high hurts performance. Lastly, EDM incorporate a particular time discretization on top of this scheduler change, potentially to compensate for the high $\sigma_{\max}$.  

% and it approximately interpolates between a standard Gaussian $\gN(0,\sigma_{\text{max}}^2 I)$ (assuming the data has $0$ mean and standard deviation much smaller than $\sigma_{\text{max}}$) and the data distribution. Using Theorem \ref{thm:equivalence} there exists an ST transformation realizing this scheduler change, and we provide the exact $(s_r$,$t_r)$ in Appendix \ref{a:edm_method}. Then, EDM suggests to solve \eqref{e:ode} with VF $\bar{u}_r(x)$ using Heun, a 2nd order RK solver and using the time discretization
% \begin{equation}
%     r_i = 1-\frac{1}{\sigma_{\max}}\parr{ \sigma_{\max}^{\frac{1}{\rho}} + \frac{i}{n-1}\parr{\sigma_{\min}^{\frac{1}{\rho}}-\sigma_{\max}^\frac{1}{\rho}}}^{\rho},
% \end{equation}
% where $\rho=7$, $n$ is the number of desired steps, $r_n=1$, and $\sigma_{\min}=0.002$. 

\textbf{Bespoke Scale-Time solvers~\cite{shaul2023bespoke}} suggest to search among the ST transformations for a particular instance that facilitates sampling a specific model. In more detail, applying an ST transformation to \eqref{e:exact_sol} provides  
\begin{equation}\label{e:st_exact_sol}
    \bar{x}(r_{i+1}) = \bar{x}(r_i) + \int_{r_i}^{r_{i+1}}\bar{u}_r(\bar{x}(r))dr,
\end{equation}
where $\bar{x}(r)$ and $\bar{u}_r(x)$ are as in equations \ref{e:bar_x} and \ref{e:bar_u} (resp.). Now applying a generic solver (Section \ref{sss:generic_solvers}) one gets an approximation to $x(1)$. Bespoke ST algorithm then searches among the space of all $(s_t,t_r)$ for the one that in expectation leads to good approximations of $x(1)$ over a set of training sample trajectories.
Another related work is \cite{watson2021learning} that also optimizes for a sampling scheduler, concentrating on discrete diffusion models and perception losses.

\paragraph{Exponential Integrator \cite{song2022denoising,zhang2022fast,lu2022dpm,lu2023dpmsolver}.} For $\eps$/$x$-prediction diffusion model with VF as defined in \eqref{e:u_gaussian_path} the sampling ODE of \eqref{e:ode} takes the form
\begin{equation}\label{e:ode_guassian_path}
    \dot{x}(t) = \frac{\dot{\psi}_t}{\psi_t}x(t) +\eta\frac{\sigma_t\dot{\alpha}_t - \dot{\sigma}_t\alpha_t}{\psi_t}f_t(x(t)),
\end{equation}
where $f_t$ is the $\eps$/$x$-prediction, and
\begin{equation}
    (\psi_t,\eta) = \begin{cases}
        (\alpha_t,-1) &\text{if $f$ is $\eps$-pred}\\
        (\sigma_t,1) &\text{if $f$ is $x$-pred}
    \end{cases}.
\end{equation}
Now changing the original model's scheduler $(\alpha_t,\sigma_t)$ to 
\begin{equation}\label{e:ei_sche_trans}
    \bar{\alpha}_t=\frac{1}{\psi_r}\alpha_r, \quad \bar{\sigma}_r=\frac{1}{\psi_r}\sigma_r,
\end{equation}
which corresponds to the conditional paths $p_t(x|x_1)=\gN(x|x_1,\mathrm{snr}(r)^{-2}I)$ for $\eps$-prediction and $p_t(x|x_1)=\gN(x|\mathrm{snr}(r)x_1,I)$ for $x$-prediction, where as before $\mathrm{snr}(r)=\alpha_r/\sigma_r$. Using \eqref{e:bar_u} in \eqref{e:st_exact_sol} and rearranging leads to Exponential Integrators' basic formula (equivalent to the result after employing variation of constants method~\cite{lu2023dpmsolver}) 
\begin{equation}\label{e:exp_exact_sol}
    x(t_{i+1}) = \frac{\psi_{t_{i+1}}}{\psi_{t_i}}x(t_i) +y\psi_{t_{i+1}} \int_{\lambda_{t_i}}^{\lambda_{t_{i+1}}}e^{y\lambda}f_{\lambda}(x(\lambda))d\lambda,
\end{equation}
where $\lambda_t=\log \mathrm{snr}(t)$ and $f_\lambda = f_{t_\lambda}$, where $t_\lambda$ is the inverse of $\lambda_t$ which is defined since we assume $\lambda_t$ is monotonically increasing. Note that this transformation is also discussed~\cite{zhang2023fast}. Now using generic solvers (Section \ref{sss:generic_solvers}) to approximate the integral above leads to the desired solver. 
%As also shown in~\cite{zhang2023fast}, \eqref{e:exp_exact_sol} can be shown to be an ST transformation of \eqref{e:exact_sol} defined by $s_r=\psi_r^{-1}$, $t_r=r$.  
% \begin{restatable}{lemma}{stsubsumeei}\label{lem:st_subsume_ei}
%     The Scale-Time solvers family subsume the Exponential Integrator family.
% \end{restatable}
% Proof in Appendix \ref{a:universality}

% We prove two theorems that delineate the NSLS family, One theorem outlines its expressiveness, while the other provides a unique representation for each of its members
We are now ready to formulate our main theorem proving the relations depicted in Figure \ref{fig:ven-diagram}, proved in Appendix \ref{a:universality}:
\begin{restatable}[Solver Taxonomy]{theorem}{taxonomy}\label{thm:universality}
    The Runge-Kutta (RK and Exponential-RK). family is included in the Scale-Time RK family, while the Multistep family (Multistep and Exponential-Multistep) is included in the Scale-Time multistep family. The Scale-Time family is included in the Non-Stationary solvers family.
\end{restatable}
% \begin{theorem}[Solver Taxonomy]\label{thm:universality}
%     The Runge-Kutta (RK and Exponential-RK) family is included in the Scale-Time RK family, while the Multistep family (Multistep and Exponential-Multistep) is included in the Scale-Time multistep family. The Scale-Time family is included in the Non-Stationary solvers family.   
%     %The NSS family subsumes Multistep, Runge-Kutta, Exponential Integrator, Scale-Time transformation, \ns{Need to decide which dedicated solvers we want to include here}. \yl{add more specific inclusion in the statement, \eg, RK$\subset$ST-RK...}
% \end{theorem}

%\section{Time Dependent Classifier-Free Guidance}
%\ns{Should we dedicate a section to explain it or just a sub section of Experiments?}

\section{Previous work}
%Here we will cover previous works not already covered in Section \ref{s:expressive_power}. 
Most previous works on \emph{dedicated solvers} and {solver distillation} are already covered in Section \ref{s:expressive_power}. Here we discuss works that are not yet covered.
\pagebreak
Another related work on solver distillation is \cite{duan2023optimal} that removes time steps from a diffusion sampler and uses a similar parameterization to \eqref{e:unique_nsms} for approximating the missing $\eps$-prediction values with linear projection. In contrast, we formulate a single optimization problem over the NS family of solvers to directly minimize the solver's error. %
% in This line of work aims to train model-specific solvers. In addition to the works described above, \cite{duan2023optimal} employs a similar parameterization to ours, utilizing linear projections and formulating a min-max problem to approximate the sampling path of DDIM with a reduced NFE. 
% \begin{itemize}
%     \item \cite{Duan_2023} uses a solver with same form as ours but consider but with a very different algorithm for optimization, need to discuss what to write.\\
%     \item \cite{watson2021learning} essentially propose to optimize the scheduler $(\alpha_t, \sigma_t)$ post training of the model. we can claim it is subsumed by ST transformation.
%     \item \citet{dockhorn2022genie}~(GENIE) introduced a higher-order solver, and distilled the necessary JVP for their method
% \end{itemize}

\textbf{Model distillation} also targets learning an efficient solver for the original model but for that end fine-tunes the original. Early attempts minimize directly a sample approximation loss~\cite{luhman2021knowledge}, while follow-up approaches progressively reduce the number of steps~\cite{salimans2022progressive, meng2023distillation}, or iteratively fine-tune from previous model's samples~\cite{liu2022flow}. We show that BNS solvers, that use only a tiny fraction of the parameter count of model distillation, can be trained quickly on a tiny training set and achieve nearly comparable sampling perceptual quality.

\input{tables/images_shutterstock}

\section{Experiments}
We evaluate BNS solvers on: (i) Class conditional image generation, (ii) Text-to-Image generation, and (iii) Text-to-Audio generation. Additionally, we compare our method with model distillation.  
%
%\paragraph{Implementation details and pretrained models.} 
Unless stated otherwise, conditional sampling is done using classifier-free guidance (CFG)~\cite{ho2022classifier,zheng2023guided}. All BNS solvers are trained on $520$ pairs $(x_0, x(1))$ of noise and generated image using adaptive RK45~\cite{shampine1986some} solver. During optimization (Algorithm \ref{alg:nsms_training}) we log PSNR on a validation set of $1024$ such pairs and report results on best validation iteration. Further details are in Appendix \ref{a:bns_implementation}. 
As pre-trained models we use: (i) $\eps$-prediction Diffusion model~\cite{ho2020denoising} with the Variance Preserving  scheduler ($\eps$-VP)~\cite{song2020score}; (ii) Flow-Matching with the Conditional Optimal-Transport scheduler (FM-OT)\cite{lipman2022flow,liu2022flow}, and Flow-Matching/$v$-prediction with the Cosine scheduler (FM/$v$)~\cite{salimans2022progressive,albergo2022si}. Details of the pre-trained models are in Appendix \ref{a:pre_trained}.\vspace{-5pt}
% \subsection{Time dependent classifier-free guidance.} \yl{discuss briefly motivation and previous methods. }\ns{Large text-to-image models are known to use CFG to improve image quality on the expense of diversity\ns{cite}. Specifically, its common practice to apply a high guidance scale~\cite{rombach2021highresolution, saharia2022photorealistic} that increase image-text alignment, while tend to results in a saturated colors. We propose to use a \emph{time dependent guidance scale}, that is
% \begin{equation}
%     u_t^{w}(x) = \parr{1+w_t}u_t(x|y) - w_t u_t(x|y),
% \end{equation}
% where $w:[0,1]\too\R$.
% \begin{itemize}
%     \item \cite{saharia2022photorealistic} suggested thresholding to handle large guidance scale.
% \end{itemize}
% }

% \subsection{BNS Optimization}
% \yl{remove and put inside the different paragraphs below}
% The NSS family has only one hyper-parameter, $n$ - the number of steps. We train for ImageNet 64/128 $n\in\set{4,6,8,...,20}$, and for Shutterstock 512 $n\in\set{12,16,20}$. The NSS optimization requires: a set of XXX pairs $\set{(x(0),x(1))}$ divided to XXX train, XXX validation used for stopping criteria. parameters initialization: we use generic solver or use the ST-formula and change path to XXX. \yl{LR, number of iterations, batch size, optimizer}
\input{tables/table_shutterstock}

\subsection{Class condition image generation.}\vspace{-5pt}
We evaluate our method on the class conditional ImageNet-64/128 ~\cite{deng2009imagenet} dataset. As recommended by the authors~\cite{imagenet_website} to support fairness and preserve people's privacy we used the official \emph{face-blurred} data, see more details in Appendix \ref{a:experiments_imagenet}. We report PSNR w.r.t.~ground truth (GT) images generated with adaptive RK45 solver~\cite{shampine1986some}, and Fréchet
Inception Distance (FID)~\citep{heusel2017GANs}, both metrics are computed on 50k samples from the models. We train our BNS solvers for NFE $\in \set{4,6,\ldots,20}$ with RK-Midpoint initial solver and preconditioning $\sigma_0=1$, each taking $0.1-1\%$ fraction of the GPU days used to train the diffusion/flow models (\ie, 2-10 GPU days with Nvidia V100). We compare our results against various baselines, including generic solvers, exponential solvers:  DDIM~\cite{song2022denoising}, and DPM\cite{zhang2023fast}, as well as the BST~\cite{shaul2023bespoke} distilled solvers. Figure \ref{fig:graph_imagenet_main} shows our BNS solvers improves both PSNR and FID over all baselines. Specifically, in PSNR metric we achieve a large improvement of at least $5-10$dB above the runner-up baseline and get to $5\%$ from FID of the GT solver (about $160-320$ NFE) with $16$ NFE. Qualitative examples are shown in Figures \ref{fig:a_images_imagenet128} and \ref{fig:a_images_imagenet64_eps_vp} in Appendix \ref{a:experiments_imagenet}. Interestingly, for PSNR we see the order: BNS $>$ BST $>$ DPM $>$ RK-Midpoint/Euler, that matches well the solver hierarchy proved in Theorem \ref{thm:universality}, see also Figure \ref{fig:ven-diagram}. In Figure \ref{fig:graph_imagenet_BNS_vs_BST} we also show an ablation experiment comparing the Non-Stationary and Scale-Time family both optimized with Algorithm \ref{alg:nsms_training},  demonstrating the benefit in the NS family of solvers over the ST family. \vspace{-5pt}

\subsection{Text-to-Image generation.}\vspace{-5pt}
In considerations regarding the training data of Stable Diffusion (such as copyright infringements and consent), we have opted not to experiment with this model. Hence, we use a large latent FM-OT T2I model (2.2b parameters) trained on a proprietary dataset of 330m image-text pairs. Image size is $512\times512\times 3$ while the latent space is of dimension $64\times64\times 4$; see implementation details in Appendix \ref{a:pre_trained}. For evaluation we report PSNR w.r.t~GT images, similar to the class conditional task. Additionally, we use MS-COCO~\cite{lin2015microsoft} validation set and report perceptual metrics including Pick Score~\cite{kirstain2023pickapic}, Clip Score~\cite{ramesh2022hierarchical}, and zero-shot FID. All four metrics are computed on 30K generated and validation images and reported for guidance (CFG) scale $w=2$ and $w=6.5$. For each guidance scale, we optimize BNS solvers for NFE $\in\set{12,16,20}$ with initial solver RK-Euler. Each solver training takes 15-24 GPU days with Nvidia V100, consisting at most $0.3\%$ fraction of the GPU days used to train the latent FM-OT model. We find that $\sigma_0=5$ gives best results for $w=2$, while $\sigma_0=10$ for $w=6.5$. As baselines we compare our results to RK-Midpoint/Euler. Table \ref{tab:shutterstock} shows BNS solvers improves PSNR by at least $10$dB and consistently improves Pick Score as well. The Clip Score and FID metrics are not correlated with NFE and are considered noisy metrics for T2I evaluations \cite{kirstain2023pickapic}. Figure \ref{fig:images_shutterstock} shows qualitative examples. Additionally, Table \ref{tab:shutterstock_initial_solver}, and Figures \ref{fig:a_images_t2i_w2.0} and \ref{fig:a_images_t2i_w6.5} in Appendix \ref{a:t2i} shows an ablation comparing BNS solver to its initialization. Lastly, we note that higher guidance scale generally tends to be hard to approximate as can be noticed by comparing PSNR values for different NFEs in Table \ref{tab:shutterstock}. 

\input{tables/bespoke_vs_distillation}

\subsection{Bespoke solvers vs. Distillation}
We compare BNS solvers with Progressive Distillation (PD)~\cite{salimans2022progressive} on two datasets: CIFAR10~\cite{krizhevsky2009learning}, and class conditional ImageNet-64\footnote{Note that BNS evaluates on models trained with the blurred face ImageNet as recommended in ImageNet website~\cite{imagenet_website} to support fairness and preserve people's privacy.}. For BNS we use the FM-OT models. For fair comparison, we report both BNS and PD in the unguided setting (\ie, $w=0$); results for PD taken from~\cite{salimans2022progressive,meng2023distillation}. Table \ref{tab:bespoke_vs_disitillation} shows FID, number of forward passes in the model during training (Forwards), where computation is detailed in Appendix \ref{a:bespoke_vs_disitillation}, training set size (Training Set), and number of trained parameters in BNS/PD (Parameters). While on NFE $<8$ we fail to compete with PD's FID, we see that in the mid range of $8-16$ NFE our BNS solver gives comparable FID using significantly less compute. Specifically, for ImageNet-64 our training uses only $0.5\%$ of the forwards used by PD. Additionally, the low number of parameters allows us to generalize well despite the tiny training set. %using only $520$ generated images as training set. 

\subsection{Audio generation.}
\input{tables/graphs_audio}

Next, we experiment with BNS solvers on an audio generation model.  We use the speech model introduced by ~\cite{vyas2023audiobox}, which is a latent Flow-Matching model trained to infill Encodec ~\cite{défossez2022high} features, conditioned on frame-aligned text transcripts. To train the BNS and BST solvers we generate 10k random samples from the training set using the RK45 solver.  
We evaluate on 8 different datasets, each of which are described in \ref{sec:app_audio_generation}.  In each setting, the model is given a transcript and a (possibly empty) audio prompt. The model needs to synthesize speech corresponding to the transcript and the speech should preserve the speaker style of the given audio prompt if one is provided.  We evaluate by computing the SNR (dB)  
w.r.t.~ground truth samples generated using the adaptive RK45 solver. \cref{fig:graphs_t2a} compares the SNR for two datasets at different NFEs for each solver.  The remaining datasets can be found in \cref{fig:a_graphs_t2a}.  Across all datasets BNS solver is consistently better than baselines improving 1dB-3dB from runner-up. 

%We additionally evaluate word error rate (WER) and speaker similarity for each solver in \ref{sec:app_audio_generation}, showing that all solvers are roughly equivalent, but noting these are noisy . 

\section{Conclusions and limitations}
We have introduced Bespoke Non-Stationary (BNS) solvers based on the provably expressive Non-Stationary (NS) solvers family and demonstrated this theoretical expressiveness translates to better samples approximation at low NFE presenting best PSNR per NFE results among a large set of baselines and applications. In contrast to previous solver distillation methods such as \cite{shaul2023bespoke} BNS don't need to a-priori fix a base solver and consequently an order, however it does need to optimize a different solver for different NFE, which opens an interesting future research question whether a single solver can handle different NFE without degrading performance. Further  limitations of BNS solvers is that they don't reach the extremely low NFE regime (1-4), and for T2I generation utilize CFG (increasing the effective batch size). An interesting future work is to further increase the expressiveness to further reduce NFE and potentially incorporate conditional guidance in the solver.

\newpage
\section{Impact Statement}
This paper presents a method for fast sampling of diffusion and flow models. There are many potential societal consequences of our work, none which we feel must be specifically highlighted here.
\section*{Acknowledgements}
NS is supported by a grant from Israel CHE
Program for Data Science Research Centers.

% \textbf{Do not} include acknowledgements in the initial version of
% the paper submitted for blind review.

% If a paper is accepted, the final camera-ready version can (and
% probably should) include acknowledgements. In this case, please
% place such acknowledgements in an unnumbered section at the
% end of the paper. Typically, this will include thanks to reviewers
% who gave useful comments, to colleagues who contributed to the ideas,
% and to funding agencies and corporate sponsors that provided financial

\bibliography{main}
\bibliographystyle{icml2024}

%%%%%%%%%%%%%%%%%%%%%%%%%%%%%%%%%%%%%%%%%%%%%%%%%%%%%%%%%%%%%%%%%%%%%%%%%%%%%%%
%%%%%%%%%%%%%%%%%%%%%%%%%%%%%%%%%%%%%%%%%%%%%%%%%%%%%%%%%%%%%%%%%%%%%%%%%%%%%%%
% APPENDIX
%%%%%%%%%%%%%%%%%%%%%%%%%%%%%%%%%%%%%%%%%%%%%%%%%%%%%%%%%%%%%%%%%%%%%%%%%%%%%%%
%%%%%%%%%%%%%%%%%%%%%%%%%%%%%%%%%%%%%%%%%%%%%%%%%%%%%%%%%%%%%%%%%%%%%%%%%%%%%%%
\newpage
\appendix
\onecolumn
\section{BNS Optimization}
\label{a:bns_optim}
\parameterization*
\begin{proof}[Proof of proposition \ref{prop:ns_param}]
    To prove the proposition we use induction on the step number $0\le i \le n-1$, where our induction hypothesis is the proposition itself. Remember, an Update rule of a NS solver represented by $(c_i,d_i)\in\R^{i+1}\times\R^{i+1}$ is
    \begin{align} 
        x_{i+1} &= X_ic_i + U_id_i\\
        &= \sum_{j=0}^i (c_{i})_jx_j + \sum_{j=0}^i (d_{i})_ju_j.
    \end{align}
    First, for the base case $i=0$, both $(c_0,d_0)\in \R\times\R$ and $(a_0,b_0)\in \R\times\R$, hence we can take $a_0=c_0$ and $b_0=d_0$. Let $k<n-1$, we assume the hypothesis is true for every $i \le k$. 
    Then we can write the $k^{\text{th}}$ step as
    \begin{align}
        x_{k+1} &= \sum_{j=0}^k (c_{k})_jx_j + \sum_{j=0}^k (d_{k})_ju_j\\
        &= (c_{k})_0x_0 + \sum_{j=0}^{k-1} (c_{k})_{j+1}x_{j+1} + \sum_{j=0}^k (d_{k})_ju_j\\
        &= (c_{k})_0x_0 + \sum_{j=0}^{k-1}(c_{k})_{j+1}\parr{a_{j}x_0+ \sum_{l=0}^{j}(b_{j})_lu_l} + \sum_{j=0}^k (d_{k})_ju_j\\
        &= \parr{(c_{k})_0+ \sum_{j=0}^{k-1} (c_{k})_ja_j}x_0 + \sum_{j=0}^{k-1} (c_{k})_{j+1} \sum_{l=0}^{j}(b_{j})_lu_l + \sum_{j=0}^k(d_{k})_ju_j\\
        &= \parr{(c_{k})_0+ \sum_{j=0}^{k-1} (c_{k})_ja_j}x_0 + \sum_{l=0}^{k-1}\sum_{j=l}^{k-1} (c_{k})_{j+1} (b_{j})_lu_l + \sum_{j=0}^k(d_{k})_ju_j\\
        &= \parr{(c_{k})_0+ \sum_{j=0}^{k-1} (c_{k})_ja_j}x_0 + \sum_{j=0}^{k-1}\sum_{l=j}^{k-1} (c_{k})_{l+1} (b_{l})_ju_j + \sum_{j=0}^k(d_{k})_ju_j\\
        &= a_{k}x_0 + \sum_{j=0}^k (b_{k})_ju_j
    \end{align}
    where in the \nth{2} equality we made the shift $j \mapsto j+1$, in the \nth{3} equality we substitute $(a_i,b_i)$ for $i\le k$ given by our induction assumption, in the \nth{5} equality we changed the order of summation to first sum on $j$ index and then on $l$, in the \nth{6} equality we only switched the notation of $l$ and $j$ indices, finally in the last equality we define
    \begin{equation}
        a_k = \parr{(c_{k})_0+ \sum_{j=0}^{k-1} (c_{k})_ja_j},\quad (b_k)_j = \sum_{l=j}^{k-1} (c_{k})_{l+1} (b_{l})_j + (d_k)_j, j=0,\ldots k-1,\quad (b_k)_k = (d_k)_k.
    \end{equation}
    Note, if the vectors $x_0, u_0,\ldots,u_k$ are linearly independent then the above coefficients, $a_k, (b_k)_0,\dots,(b_k)_k$ are unique.
\end{proof}
\section{Universality Of Non-Stationary Solvers}
\label{a:universality}
This appendix provides a proof of theorem \ref{thm:universality} and its visualization in the Ven diagram in figure \ref{fig:ven-diagram}. We state the first part of the theorem in lemma \ref{lem:st_subsume_genric_ei} and provide a stand-alone proof of the lemma. Then using lemma \ref{lem:st_subsume_genric_ei} we complete the proof of theorem \ref{thm:universality}.
\begin{lemma}\label{lem:st_subsume_genric_ei}
     The Runge-Kutta (RK and Exponential-RK) family is included in the Scale-Time RK family, while the Multistep family (Multistep and Exponential-Multistep) is included in the Scale-Time multistep family.
\end{lemma}
\begin{proof}[Proof of lemma \ref{lem:st_subsume_genric_ei}]
    Remeber, given a pair of Scale-Time (ST) transformation $(s_r,t_r)$ and a generic solver, its associated solver in the ST  solver family is defined as an approximation using the generic solver to the exact solution as in \eqref{e:exact_sol} for the transformed VF,
\begin{equation}\label{e:a_st_exact_sol}
    \bar{x}(r_{i+1}) = \bar{x}(r_i) + \int_{r_i}^{r_{i+1}}\bar{u}_r(\bar{x}(r))dr,
\end{equation}
where $\bar{x}(r)$ and $\bar{u}_r(x)$ are as in equations \ref{e:bar_x} and \ref{e:bar_u} (resp.). First consider the identity transformation as the ST transformation, that is
\begin{equation}
    s_r = 1, \quad t_r = r,
\end{equation}
then \eqref{e:a_st_exact_sol} coincides with \eqref{e:exact_sol}. In this case, applying the generic solver gives the solver itself, hence Multistep family is included in the ST Multistep family and RK family is included in ST RK family. Next, we consider an Exponential Integrator, it is defined as an approximation to the exact solution~\cite{lu2023dpmsolver,lu2022dpm-pp} as in \eqref{e:exp_exact_sol},
    \begin{equation}\label{e:a_exp_exact_sol}
    x(t_{i+1}) = \frac{\psi_{t_{i+1}}}{\psi_{t_i}}x(t_i) +\eta\psi_{t_{i+1}} \int_{\lambda_{t_i}}^{\lambda_{t_{i+1}}}e^{\eta\lambda}f_{\lambda}(x(\lambda))d\lambda,
\end{equation}
where $\lambda_t=\log \mathrm{snr}(t)$ and $f_\lambda = f_{t_\lambda}$, where $t_\lambda$ is the inverse of $\lambda_t$ which is defined since we assume $\lambda_t$ is monotonically increasing, and $\psi_t$ and $\eta$ are dependent on the scheduler and the objective $f$, 
\begin{equation}
    (\psi_t,\eta) = \begin{cases}
        (\alpha_t,-1) &\text{if $f$ is $\eps$-pred}\\
        (\sigma_t,1) &\text{if $f$ is $x$-pred}
    \end{cases}.
\end{equation}
Hence, it is enough to show there exist a ST transformation $(s_r,t_r)$, such that \eqref{e:a_exp_exact_sol} and \eqref{e:a_st_exact_sol} coincide. As mention in section \ref{s:expressive_power} in \eqref{e:ei_sche_trans} we consider the change of scheduler to
\begin{equation}
    \bar{\alpha}_t=\frac{1}{\psi_r}\alpha_r, \quad \bar{\sigma}_r=\frac{1}{\psi_r}\sigma_r.
\end{equation}
By \eqref{e:conversion} its corresponding ST transformation is
\begin{equation}
    s_r = \frac{1}{\psi_r}, \quad t_r = r,
\end{equation}
and by equations \ref{e:bar_x} and \eqref{e:bar_u} the transformed trajectory and VF are
\begin{equation}\label{e:a_exp_bar_x_u}
    \bar{x}(r) = \frac{x(r)}{\psi_r},\quad \bar{u}_r(x) = -\frac{\dot{\psi}_r}{\psi_r}x + \frac{1}{\psi_r}u_r(\psi_rx).
\end{equation}
For an objective $f$ either $\eps$-prediction or $x$-prediction, the VF u is as in \eqref{e:u_gaussian_path},
\begin{equation}\label{e:a_exp_u}
    u_r(x) = \frac{\dot{\psi}_r}{\psi_r}x +\eta\frac{L_r}{\psi_r}f_r(x)
\end{equation}
where $L_r = \sigma_r\dot{\alpha}_r - \dot{\sigma}_r\alpha_r$. Finally, substitute equations \ref{e:a_exp_bar_x_u} and \ref{e:a_exp_u} into \eqref{e:a_st_exact_sol} gives
\begin{align}
    \frac{x_{r_{i+1}}}{\psi_{r_{i+1}}} &= \frac{x_{r_{i}}}{\psi_{r_{i}}} + \eta\int_{r_i}^{r_{i+1}}\frac{L_r}{\psi_r^2}f_r(x)dr\\
    &= \frac{x_{r_{i}}}{\psi_{r_{i}}} + \eta\int_{r_i}^{r_{i+1}}\frac{d}{dr}\frac{1}{\eta}\parr{\frac{\alpha_r}{\sigma_r}}^{\eta}f_r(x)dr\\
    & = \frac{x_{r_{i}}}{\psi_{r_{i}}} + \int_{r_i}^{r_{i+1}}\frac{d}{dr}\parr{e^{\eta\lambda_r}}f_r(x)dr\\
    & = \frac{x_{r_{i}}}{\psi_{r_{i}}} + \int_{r_i}^{r_{i+1}}\eta\frac{d\lambda}{dr}e^{\eta\lambda_r}f_r(x)dr\\
    & = \frac{x_{r_{i}}}{\psi_{r_{i}}} + \eta\int_{\lambda_{r_i}}^{\lambda_{r_{i+1}}}e^{\eta\lambda}f_{\lambda}(x)d\lambda,
\end{align}
where in the \nth{2} equality we notice that $\frac{L_r}{\psi_r^2}=\frac{d}{dr}\frac{1}{\eta}\parr{\frac{\alpha_r}{\sigma_r}}^{\eta}$, in the \nth{3} equality we substitute $\lambda_r = \log(\alpha_r/\sigma_r)$,  in the \nth{4} equality used the chain rule to differentiate w.r.t. $r$,in the \nth{5} equality we changed the integration variable to $\lambda$, and  multiplied both sides by $\psi_{r_{i+1}}$ gives \eqref{e:a_exp_exact_sol}. 
\end{proof}
We are ready prove the main theorem:
\taxonomy*
\begin{proof}[Proof of lemma \ref{thm:universality}.]
    By lemma \ref{lem:st_subsume_genric_ei} it is left to show that the Non-Stationary (NS) solvers family includes the Scale-Time (ST) solvers family. Remember, given an ST transformation $(s_r,t_r)$, its associated solver is defined as an approximation using a generic solver to the exact solution as in \eqref{e:exact_sol} for the transformed VF. That is,
\begin{equation}\label{e:aa_st_exact_sol}
    \bar{x}(r_{i+1}) = \bar{x}(r_i) + \int_{r_i}^{r_{i+1}}\bar{u}_r(\bar{x}(r))dr,
\end{equation}
where $\bar{x}(r)$ and $\bar{u}_r(x)$ are as in equations \ref{e:bar_x} and \ref{e:bar_u} (resp.), and the generic solvers we consider are either a Multistep or RK method. Note by equations \ref{e:multistep}, \ref{e:def_rk_step}, and \ref{e:def_rk_stage} the update rules of both Multistep and RK methods are expressed as a linear combination of $x_i$ and $u_i$, hence they are included in the NS solver family. That is, for every such generic solver with $n$ steps there exists $,\bar{a}_i,\bar{b}_i\in\R^{i+1}$, $i=0,\ldots,n-1$, and a discretization $0=r_0,r_1,\ldots,r_n=1$ such that the ST solver update rule is
\begin{align}\label{e:a_st_step}
    \bar{x}_{r_{i+1}} &= \sum_{j=0}^{i}\bar{a}_{ij}\bar{x}_{r_j} + \sum_{j=0}^{i}\bar{b}_{ij}\bar{u}_{r_j}(\bar{x}_{r_j}).
\end{align}
We substitute the definition of $\bar{x}(r)$, $\bar{u}_r(x)$, and divide both sides of \eqref{e:a_st_step} by $s_{i+1}=s_{r_{i+1}}$,
\begin{align}
    x_{t_{i+1}} &= \sum_{j=0}^{i}\frac{\bar{a}_{ij}s_i}{s_{i+1}}x_{t_j} + \sum_{j=0}^{i}\frac{\bar{b}_{ij}}{s_{i+1}}\parr{\dot{s}_jx_{t_j} + \dot{t}_j s_ju_{t_j}(x_{t_j})}\\
    &= \sum_{j=0}^{i}\parr{\frac{\bar{a}_{ij}s_i}{s_{i+1}} +\frac{\bar{b}_{ij}\dot{s}_j}{s_{i+1}} }x_{t_j} + \sum_{j=0}^{i}\frac{\bar{b}_{ij}}{s_{i+1}\dot{t}_j s_j}u_{t_j}(x_{t_j})\\
    &= \sum_{j=0}^{i}a_{ij}x_{t_j} + \sum_{j=0}^{i}b_{ij}u_{t_j}(x_{t_j}),
\end{align}
where we denoted $t_{r_j}=t_j$ and we set
\begin{equation}
    a_{ij} = \parr{\frac{\bar{a}_{ij}s_i}{s_{i+1}} +\frac{\bar{b}_{ij}\dot{s}_j}{s_{i+1}} }, \quad b_{ij} = \frac{\bar{b}_{ij}}{s_{i+1}\dot{t}_j s_j}.
\end{equation}
\end{proof}
\section{Generic solvers}
\label{a:generic_solvers}

\paragraph{Adam-Bashforth and Multistep solvers.} The Adam-Bashforth (AB) solver is derived by replacing $u_t(x(t))$ in the integral in \eqref{e:exact_sol} with an interpolation polynomial $q(t)$ constructed with the $m$ previous data points $(t_{i-m+j},u_{i-m+j})$, $j=1,\ldots,m$. Integrating $q(t)$ over $[t_{i},t_{i+1}]$ leads to an $m$-step \emph{Adam-Bashforth} (AB) update formula: 
\begin{equation}
    x_{i+1} = x_{i-m+1} + h\sum_{j=1}^m b_j u_{i-m+j},
\end{equation}
where $h=t_{i+1}-t_i$. A general (stationary) $m$-step \emph{Multistep method} is defined with the more general update rule incorporating arbitrary linear combinations of previous $x_i,u_i$:
\begin{equation}\label{e:multistep}
    x_{i+1} = \sum_{j=1}^{m} a_j x_{i-m+j} + h\sum_{j=1}^{m}b_j u_{i-m+j},
\end{equation}
where $a_j,b_j\in \R, \ j=0,\ldots,m-1$ are constants (i.e., independent of $i)$.

\paragraph{Runge-Kutta.} This class of solvers approximates the integral of $u_t(x(t))$ in \eqref{e:exact_sol} with a quadrature rule using interior \emph{nodes} in the interval $[t_i,t_{i+1}]$. Namely, it uses the data points the data points $(t_{i}+hc_j,u_{t_{i}+hc_j}(\xi_{j}))$, where $c_j$ define the RK nodes, and $\xi_j\approx x(t_{i}+hc_j)$, $j=0,\ldots,m-1$. This leads to an update rule of the form
\begin{align}
        x_{i+1} &= x_{i} + h \sum_{j=0}^{m-1}b_ju_{t_i + hc_{j}}(\xi_j),\label{e:def_rk_step}\\
        \xi_j &= \begin{cases}
            x_i & j=0\label{e:def_rk_stage}\\
            x_i + h\sum_{k=0}^{j-1}a_{jk}u_{t_i + hc_{k}}(\xi_k) &  j>0 
        \end{cases},
    \end{align}
where the matrix $a\in\R^{m\times m}$ with $a_{jk}=0$ for $j\le k$ is called the RK matrix, and $b \in \R^m$ is the RK weight vector, both independent of $i$, \ie, stationary.

\input{tables/images_shutterstock_w2.0_appendix}
\input{tables/images_shutterstock_w6.5_appendix}

\input{tables/images_imagenet128}
\input{tables/images_imagenet64_eps_vp}
\section{Experiments}
\subsection{Bespoke Non-Stationary training details}
In this section we provide the training details of the BNS solvers for the three tasks: (i) class conditional image generation, (ii) Text-to-Image generation, (iii) Text-to-Audio generation.  For all tasks training set and validation set were generate using using adaptive RK45 solver, optimization is done with Adam optimizer~\cite{kingma2017adam} and results are reported on best validation iteration.

\paragraph{Class condition image generation.} For this task we generated $520$ pairs of $(x_0,x(1))$, noise and image,  for the training set, and $1024$ such pairs for the validation set. For each model on this task, ImageNet-64 $eps$-VP/FM$v$-CM/FM-OT, and ImageNet-128 FM-OT, we train BNS solvers with NFE $\in\set{4,6,8,10,12,14,16,18,20}$. We use  with learning rate of $5e^{-4}$, a polynomial decay learning rate scheduler, batch size of $40$, for $15k$ iterations. We compute PSNR on the validation set every $100$ iterations.

\paragraph{Text-to-Image.} For this task, we generate two training and validation sets of $520$ and $1024$ pairs (resp.), one for guidance scale $w=2.0$ and one for $w=6.5$. The text prompts for the generation were taken from the training set of MS-COCO~\cite{lin2015microsoft}. For each guidance scale we train BNS solvers with NFE $\in\set{12,16,20}$, learning rate of $1e^{-4}$, cosine annealing learning rate scheduler, batch size of $8$, for $20k$ iterations. We compute PSNR on the validation set every $200$ iterations.

\paragraph{Text-to-Audio.} We generate a training set of $10k$ pairs and a validation set of $1024$ pairs. We train BNS solver with NFE $\in\set{8,12,16,20}$, and optimize with learning rate of $1e^{-4}$, cosine annealing learning rate scheduler, batch size of $40$, for $15k$ iterations. We compute SNR on the validation set every $5k$ itrations.

\label{a:bns_implementation}
\subsection{Class condition image generation}\label{a:experiments_imagenet}
\input{tables/graphs_family_ablation}

\paragraph{Dataset and implementation details.}
As recommended by the authors~\cite{imagenet_website} we used the official \emph{face-blurred} data. Specifically, for the 64$\times$64 we downsample using the open source preprocessing scripts from~\cite{chrabaszcz2017downsampled}). For ImageNet-64 we use three models as described in Appendix \ref{a:pre_trained}, while for ImageNet-128 we only use FM-OT model due to computational constraints (training requires close to 2000 NVidia-V100 GPU days).

\input{tables/table_imagenet_appendix}
\subsection{Text-to-Image}
\label{a:t2i}
\input{tables/table_shutterstock_appendix}
In this section we compare our BNS solver with the initial solver. That is, the solver used in initialization of BNS optimization. Table \ref{tab:shutterstock_initial_solver} shows PSNR, Pick Score, Clip Score, and FID of the initial solver - RK-Euler with preconditioning $\sigma_0=5$ for guidance scale $w=2$ and $\sigma_0=10$ for guidance scale $w=6.5$, and the BNS solvers.
\subsection{Bespoke vs Distillation}
\label{a:bespoke_vs_disitillation}
We compare Bespoke with Progressive Distillation (PD)~\cite{salimans2022progressive}, for CIFAR10 we compare against results reported by \cite{salimans2022progressive} and for ImageNet 64 against results reported by \cite{meng2023distillation}. To count the number of forwards in the network that was done in training of Bespoke or PD, we count a forward in the model with a batch of $1$ as one forward. The training of PD for CIFAR10 model with $8$ and $4$ steps done by \cite{salimans2022progressive}  with $500k$ and $550k$ parameters updates (reps.), each update computed a batch of 128 images and requires two evaluation of the teacher model and one evaluation of student model, which sums to $192m$ and $211m$ forwards (resp.) The training of Bespoke for CIFAR10 with $8$ and $4$ steps was done with $30k$ parameters updates and batch of 40 for both, each update requires $8$ and $4$ evaluation of the model (resp.). For Bespoke we also take in account the cost of generating the training set that cost $85k$ forwards,  which in total sums to $9.7m$ and $4.9m$ forwards (resp.). For ImageNet 64 we compare against the unguided single-$w$ model trained by \cite{meng2023distillation}. The PD training of this model with $16$, $8$, and $4$ steps is done with $300k$, $350k$, and $400k$ parameters updates (resp.) and a batch of 2048, taking in account both teacher and student models evaluation gives $1843m$, $2150m$, and $2457m$ forwards (resp.). The Bespoke training for $16$, $8$, and $4$ was done with $15k$ parameters updates and a batch of $40$, each update requires $16$, $8$, and $4$ evaluation of the model (resp.), and the cost of generating the training set is $90k$ forwards, which in total sums to $9.7m$, $4.9m$, and $2.5m$ forwards (resp.).

\subsection{Audio Generation}

\label{sec:app_audio_generation}
\input{tables/graphs_audio_appendix}

The audio generation model was evaluated on the following datasets:

\begin{itemize}
    \item LibriSpeech (test-clean): audio book recordings that are scripted and relatively clean ~\cite{Panayotov2015LibrispeechAA}
    \item CommonVoice v13.0: sentences read by volunteers worldwide.  Covers a broader range of accents and are nosier compared to LibriSpeech \cite{Ardila2019CommonVA}
    \item Switchboard: a conversational speech corpus ~\cite{godfrey1992switchboard}
    \item Expresso: A multispeaker expressive speech dataset covering 7 different speaking styles. ~\cite{nguyen2023expresso}
    \item Accent: An internal expressive and accented dataset.
    \item Audiocaps: A subset of the AudioSet dataset. 
 Contains sound sourced from YouTube videos.  ~\cite{kim2019audiocaps}
    \item Spotify: podcast recordings \cite{clifton2020spotify}
    \item Fisher: conversational speech data ~\cite{cieri2005fishers}
\end{itemize}

\input{tables/audio_metrics}

Additionally we evaluate the solvers using word error rate (WER), and speaker similarity. For WER the generated audio is transcribed using Whisper ~\cite{Radford2022RobustSR} and then WER is computed against the transcript used to generate the audio.  We quantify the speaker similarity by embedding both the audio prompt and the generated audio using WavLM-TDCNN \cite{Chen2021WavLMLS}, and compute the cosine similarity between the embeddings.  In general these metrics do not accurately reflect the sample quality of different solvers.  In instances where a solver generates a low-quality sample we qualitatively find that the speaker still sounds the same and the audio is intelligible, but there are artifacts in the audio such as static, background noise, etc. which are are not quantified by speaker similarity or WER.  As can be seen from \cref{tab:audio_ss} and \cref{tab:audio_wer} there is little variance in these metrics across solvers.  

\paragraph{Conditioning of the audio model.}
The model takes in three tensors, all of the same length: a noise tensor and conditioning which is constructed of a masked Encodec features and frame-aligned token embeddings.  These get concatenated together channel-wise, and input to the model to produce the resulting Encodec features for the entire sequence.  This is then fed to the Encodec decoder to produce the final waveform.  

\section{Pre-trained models}
\label{a:pre_trained}
In this section describe the training objective that pre-trained model we used were trained with and their schedulers. In addition, we provide architecture details for our CIFAR10, ImageNet, and Text-to-Image models. 
\paragraph{Training obejective and schedulers.} The FM-OT and FM/$v$-CS model where trained with Conditional Flow Matching (CFM) loss derived in \cite{lipman2022flow}. That is,
\begin{equation}
    \gL_{\text{CFM}}(\theta) = \E_{t, p_0(x_0), q(x_1)}\norm{u_t(x_t;\theta) - (\dot{\sigma}_tx_0 + \dot{\alpha_t}x_1)}^2,
\end{equation}
where $t$ is uniform on $[0,1]$, $p_0(x_0)=\gN\parr{x_0|0,I}$, $q(x_1)$ is the data distribution, $u_t$ is the network, $(\alpha_t, \sigma_t)$ is the scheduler, and $x_t=\sigma_tx_0 + \alpha_tx_1 \sim p_t(x|x_1)$ as in \eqref{e:def_gaussian_path}. The FM-OT scheduler is
\begin{equation}
    \alpha_t=t, \quad \sigma_t=1-t,
\end{equation}
and the FM/$v$-CS scheduler is 
\begin{equation}
    \alpha_t=\sin\frac{\pi}{2}t, \quad \sigma_t=\cos\frac{\pi}{2}t.
\end{equation}
The $\eps$-VP model was trained on a different objective, the noise prediction loss as in \cite{ho2020denoising} and \cite{song2020score} with the VP scheduler. That is,
\begin{equation}
    \gL_{\text{noise}}(\theta) = \E_{t, p_0(x_0), q(x_1)}\norm{\epsilon_t(x_t;\theta) -x_0}^2,
\end{equation}
where $t,p_0(x_0), q(x_1), x_t$ as above, $\eps_t$ is the network and the VP scheduler is
\begin{equation}
    \quad \alpha_t= \xi_{1-t}, \quad \sigma_t = \sqrt{1-\xi_{1-t}^2},\quad \xi_{s}= e^{-\frac{1}{4}s^2(B-b) -\frac{1}{2}sb},
\end{equation}
where $B=20,\ b=0.1$. 

\paragraph{Architecture details.}
The Text-to-Image model has the same architecture as used by Dalle-2\cite{ramesh2022hierarchical} (2.2b parameters) with the following changes: we use the T5 text encoder~\cite{raffel2020exploring}, we have $4$ input/output channels, finally we also have an autoencoder with the same architecture of Stable Diffusion autoencoder~\cite{rombach2021highresolution}. Our CIFAR10, and class conditional ImageNet models have the U-Net architecture as in \citet{dhariwal2021diffusion}, with the hyper-parameters listed in Table \ref{tab:training_hyper-params}.
\input{tables/table_fm_training_hyper_parameters}

\end{document}

%% file: math_commands.tex
\usepackage{xcolor}

\newcommand*{\vertbar}{\rule[-0.25ex]{0.5pt}{1.5ex}}

\newcommand{\norm}[1]{\left\Vert#1\right\Vert}

\newcommand{\set}[1]{\left\{#1\right\}}
\newcommand{\parr}[1]{\left (#1\right )}
\newcommand{\brac}[1]{\left [#1\right ]}

\newcommand{\Real}{\mathbb R}

\newcommand{\eps}{\varepsilon}
\newcommand{\too}{\rightarrow}

\definecolor{mygray}{gray}{0.95}

\usepackage{xspace}
\newcommand*{\eg}{{\it e.g.}\@\xspace}
\newcommand*{\ie}{{\it i.e.}\@\xspace}
\usepackage{amsmath,amsfonts,bm}

\def\eqref#1{equation~\ref{#1}}
\def\1{\bm{1}}

\def\eps{{\epsilon}}

\DeclareMathAlphabet{\mathsfit}{\encodingdefault}{\sfdefault}{m}{sl}
\SetMathAlphabet{\mathsfit}{bold}{\encodingdefault}{\sfdefault}{bx}{n}

\def\gD{{\mathcal{D}}}

\def\gL{{\mathcal{L}}}

\def\gN{{\mathcal{N}}}

\newcommand{\E}{\mathbb{E}}

\newcommand{\R}{\mathbb{R}}

\usepackage{wrapfig}
\usepackage{tabularx,ragged2e,booktabs,caption}
\newcolumntype{C}[1]{>{\Centering}m{#1}}

\newcolumntype{Z}[1]{>{\Left}m{#1}}

%% file: tables/teaser.tex
\begin{figure*}[t]
    \centering
    \begin{tabular}{@{\hspace{0pt}}c@{\hspace{5pt}}c@{\hspace{5pt}}c@{\hspace{5pt}}c@{\hspace{0pt}}}
    {\ \scriptsize GT (Adaptive RK45) $\approx 160$ NFE} & {\scriptsize Bespoke Non-Stationary: 16 NFE} & {\ \scriptsize RK-Midpoint: 16 NFE}  & {\ \scriptsize RK-Euler: 16 NFE} \\
    \includegraphics[width=0.22\textwidth]{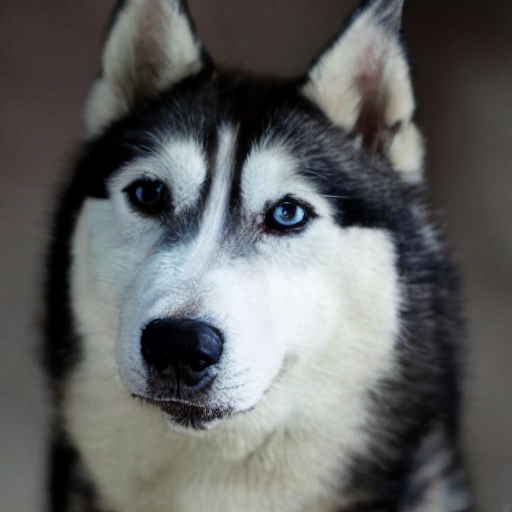}  & \includegraphics[width=0.22\textwidth]{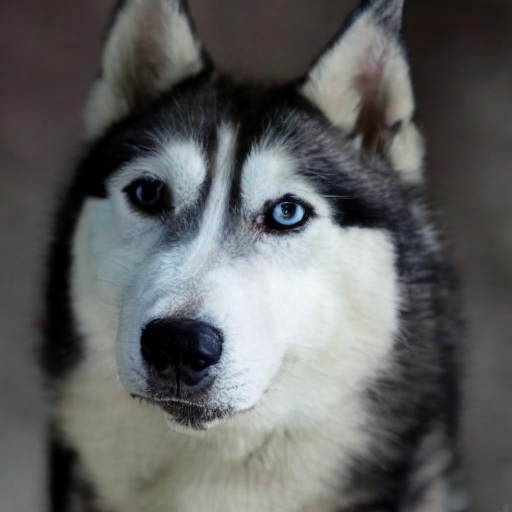} & \includegraphics[width=0.22\textwidth]{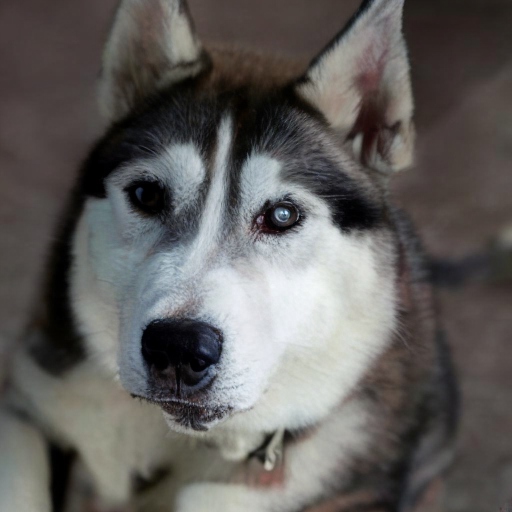} &
     \includegraphics[width=0.22\textwidth]{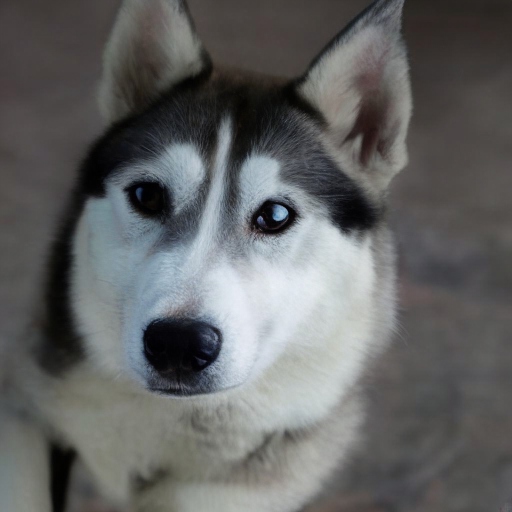}\\     
    \includegraphics[width=0.22\textwidth]{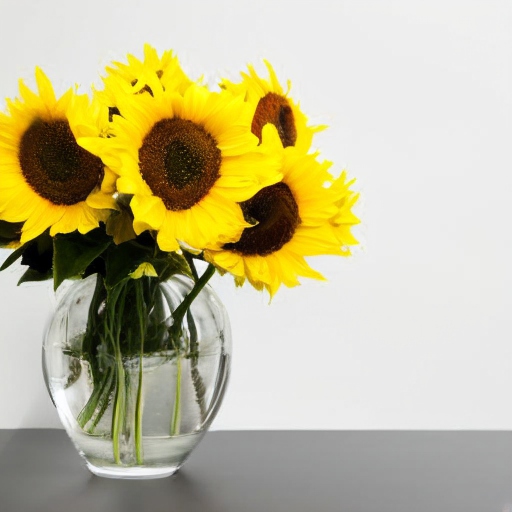} & \includegraphics[width=0.22\textwidth]{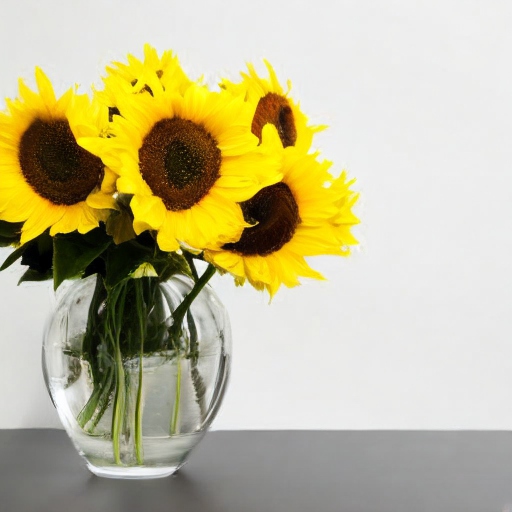} & \includegraphics[width=0.22\textwidth]{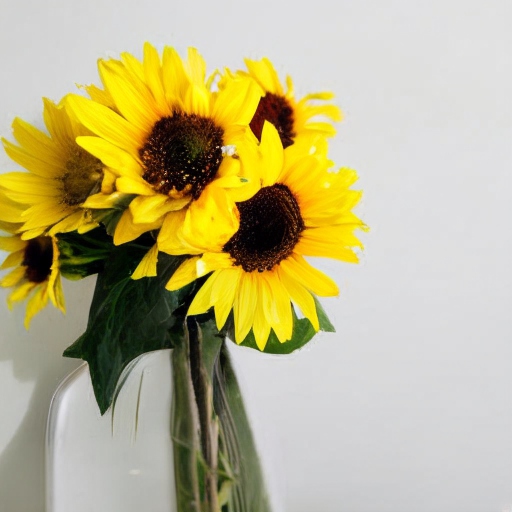} &
    \includegraphics[width=0.22\textwidth]{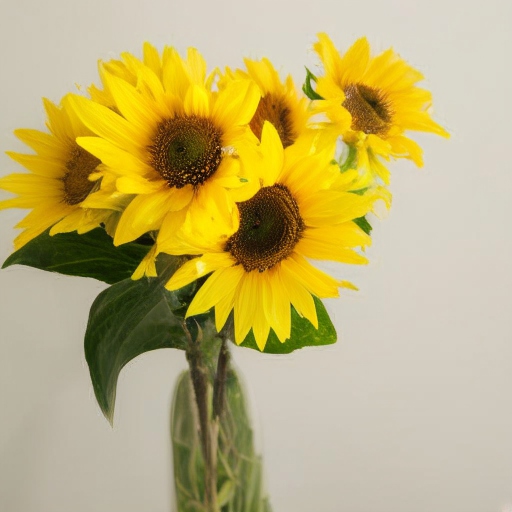}\\
    \includegraphics[width=0.22\textwidth]{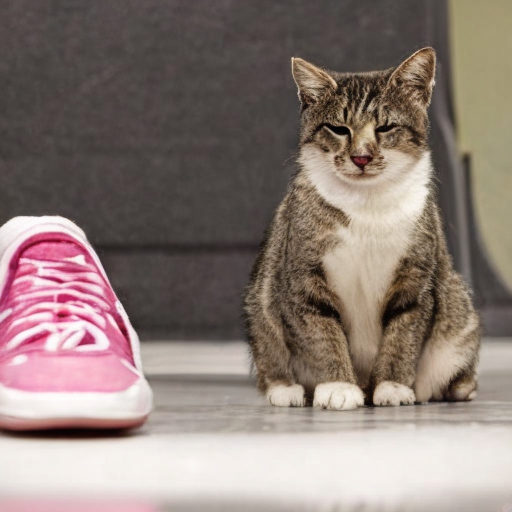} & \includegraphics[width=0.22\textwidth]{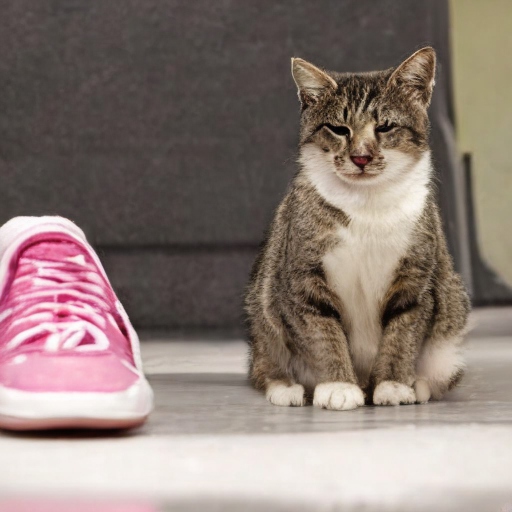} & \includegraphics[width=0.22\textwidth]{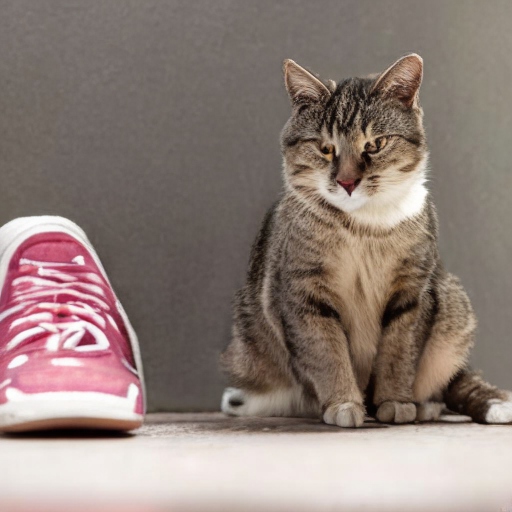} &
    \includegraphics[width=0.22\textwidth]{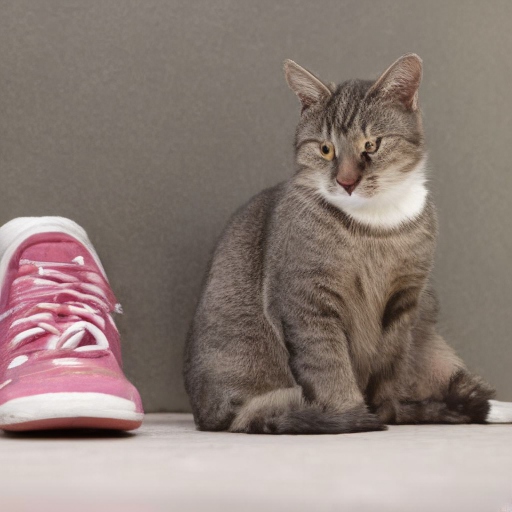}
\end{tabular}    
     \caption{Different solvers on an FM-OT 512$\times$512 Text-to-Image model with guidance scale $2$ initiated with the same noise (from left to right): Ground truth (Adaptive RK45), BNS 16 NFE (this paper), RK-Midpoint 16 NFE, and RK-Euler 16 NFE. Note the fidelity of BNS compared to GT. The different rows correspond to the captions (top to bottom): \textit{"A husky facing the camera.", "sunflowers in a clear glass vase on a desk.", "the cat is sitting on the floor beside a pair of tennis shoes."}.
     } 
    \label{fig:teaser}
\end{figure*}

%% file: tables/velocity_field_coefficients.tex
\begin{wraptable}[8]{r}{0.20\textwidth}
\vspace{-10pt}
\hspace{-15pt}
\setlength{\tabcolsep}{2.0pt}
\centering
\resizebox{0.2\textwidth}{!}{
    \begin{tabular}{lcc} 
     %\Xhline{1pt} %\hline 
      \toprule
       $f_t$ &  $\beta_t$ & $\gamma_t$ \\ \midrule %\Xhline{1pt}
      \textbf{Velocity} & $0$ & $1$ \\ \midrule
      $\boldsymbol{\eps}$\textbf{-pred} & $\frac{\dot{\alpha}_t}{\alpha_t}$ & $\frac{\dot{\sigma}_t\alpha_t - \sigma_t\dot{\alpha}_t}{\alpha_t}$\\ \midrule
      $\boldsymbol{x}$\textbf{-pred} & $\frac{\dot{\sigma}_t}{\sigma_t}$ & $\frac{\sigma_t\dot{\alpha}_t - \dot{\sigma}_t\alpha_t}{\sigma_t}$ \\
     %\Xhline{1pt} 
     \bottomrule 
    \end{tabular}  
} \vspace{-10pt}
\caption{Velocities \\ of common models.}\label{tab:velocity_field_coefficients}
\end{wraptable}%

%% file: algorithms/nsms_sampling.tex
% \begin{wrapfigure}[7]{r}{0.3\textwidth}
% \vspace{-23pt}
% \resizebox{0.289\textwidth}{!}{
%     \begin{minipage}{0.37\textwidth}
      \begin{algorithm}[H]
        \caption{Non-Stationary sampling.}\label{alg:nsms_sampling}
        \begin{algorithmic}
            \Require NS solver $\theta$, model $u$, initial noise $x_0$
            %\State {\small \color{gray} \% Sequential for loop over steps}
            \State $U_{-1} \gets \brac{\ }$\Comment{{\color{cyan}empty matrix initialization}}
            \For{$i=0,1,\ldots,n-1$}
            \State $U_i \gets \begin{bmatrix}U_{i-1} & u_{t_i^{\theta}}(x_{i})\end{bmatrix}$
            \State $x_{i+1} \gets x_0a_i^{\theta} + U_ib_i^{\theta}$
        \EndFor
        \State \Return $x_n^{\theta}$
        \end{algorithmic}
      \end{algorithm}
%     \end{minipage} 
%     }
% \end{wrapfigure}

%% file: algorithms/nsms_training.tex
% \begin{wrapfigure}[20]{r}{0.3\textwidth}
% %\vspace{-23pt}
% \resizebox{0.289\textwidth}{!}{
%     \begin{minipage}{0.37\textwidth}
      \begin{algorithm}[H]
        \caption{Bespoke Non-Stationary solver training.}\label{alg:nsms_training}
        \begin{algorithmic}
            \Require model $u$, pairs $\gD=\set{(x_0,x(1))}$, $n$, $\theta_0$
            \State initialize $\theta \gets \theta_0$ 
            \While{not converged}
                \For{$(x_0, x(1))\in \gD$}
                \State $x_n^{\theta}\gets \texttt{NS\_sampling}(x_0,u,\theta)$ \Comment{{\color{cyan}Alg.~\ref{alg:nsms_sampling}}}
                %\State $\mathcal{L}(\theta) = \text{PSNR}\parr{x_n^{\theta},x(1)}$
                \State $\theta \gets \theta - \gamma \nabla_\theta\gL(\theta)$  \Comment{{\color{cyan}optimization step, eq.~\ref{e:loss}}}
                \EndFor
            \EndWhile
        \State \Return $\theta$
        \end{algorithmic}
      \end{algorithm}
%     \end{minipage} 
%     }
% \end{wrapfigure}

%% file: tables/graphs_imagenet_main.tex
\begin{figure*}[t]
    \centering
    \begin{tabular}{@{\hspace{0pt}}c@{\hspace{0pt}}c@{\hspace{0pt}}c@{\hspace{0pt}}c@{\hspace{0pt}}}
    {\quad \ \scriptsize ImageNet-64: $\eps$-pred} & {\quad \ \scriptsize ImageNet-64: FM/$v$-CS} & {\quad \ \scriptsize ImageNet-64: FM-OT}  & {\quad \ \scriptsize ImageNet-128: FM-OT} \\
    \includegraphics[width=0.25\textwidth]{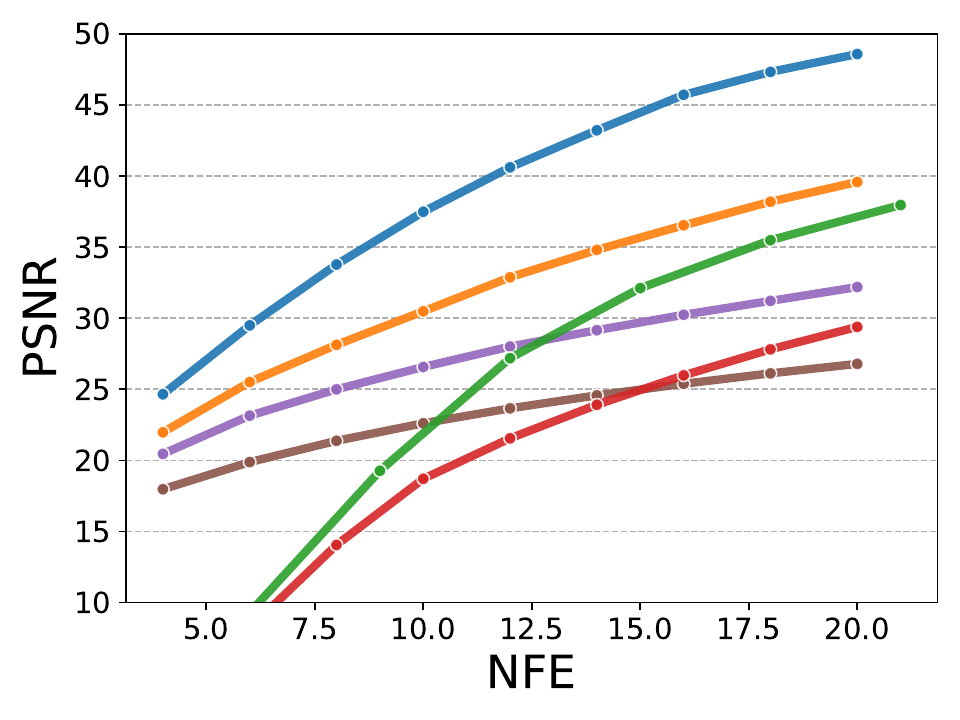}  & \includegraphics[width=0.25\textwidth]{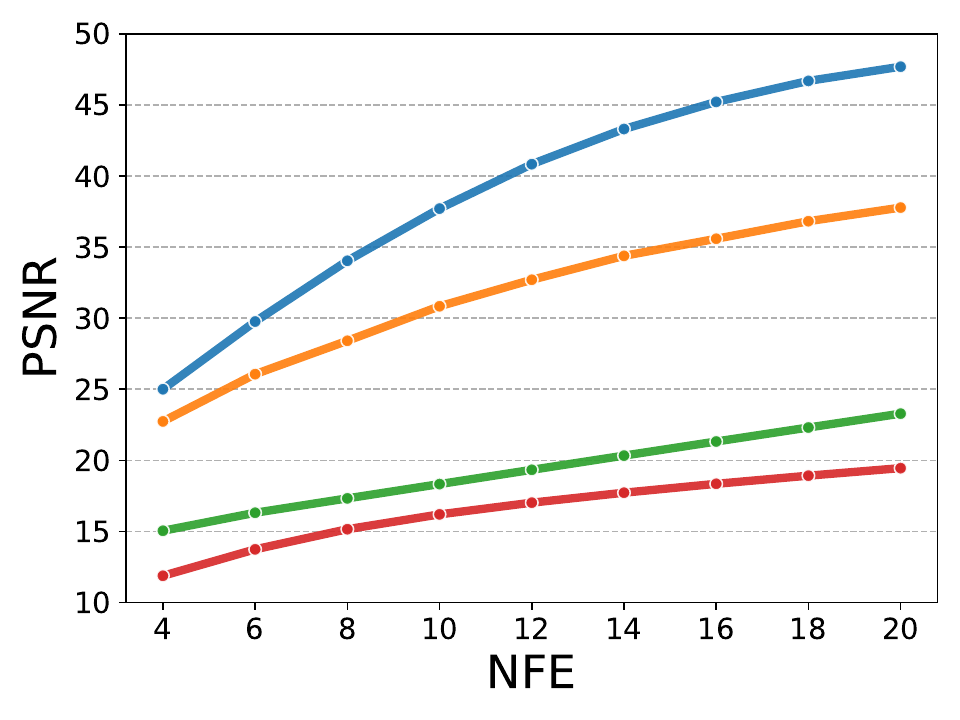} & \includegraphics[width=0.25\textwidth]{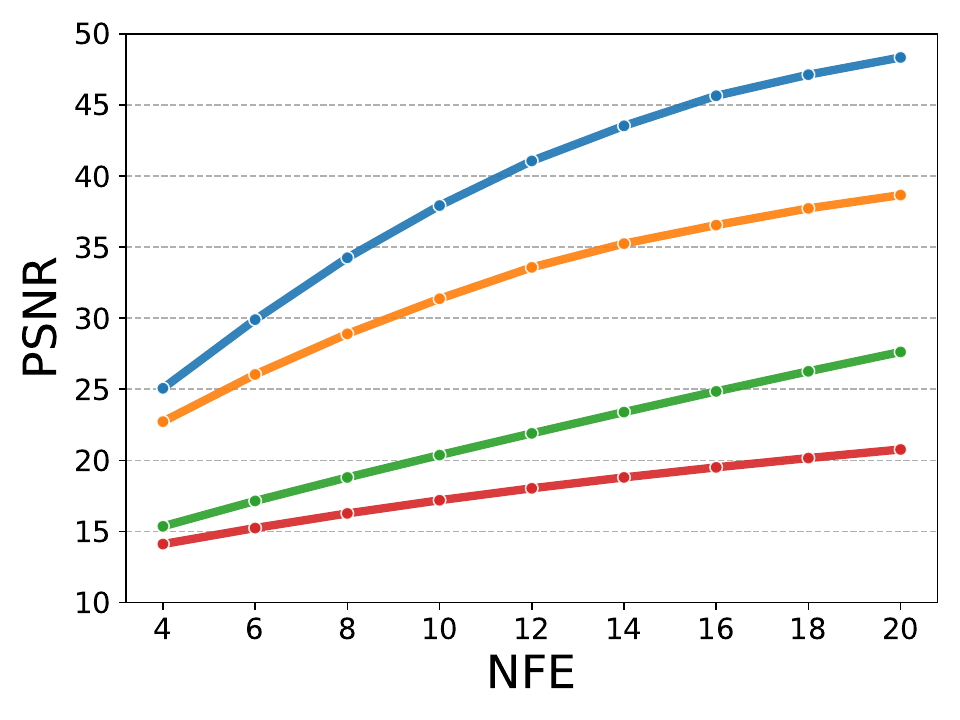} &
     \includegraphics[width=0.25\textwidth]{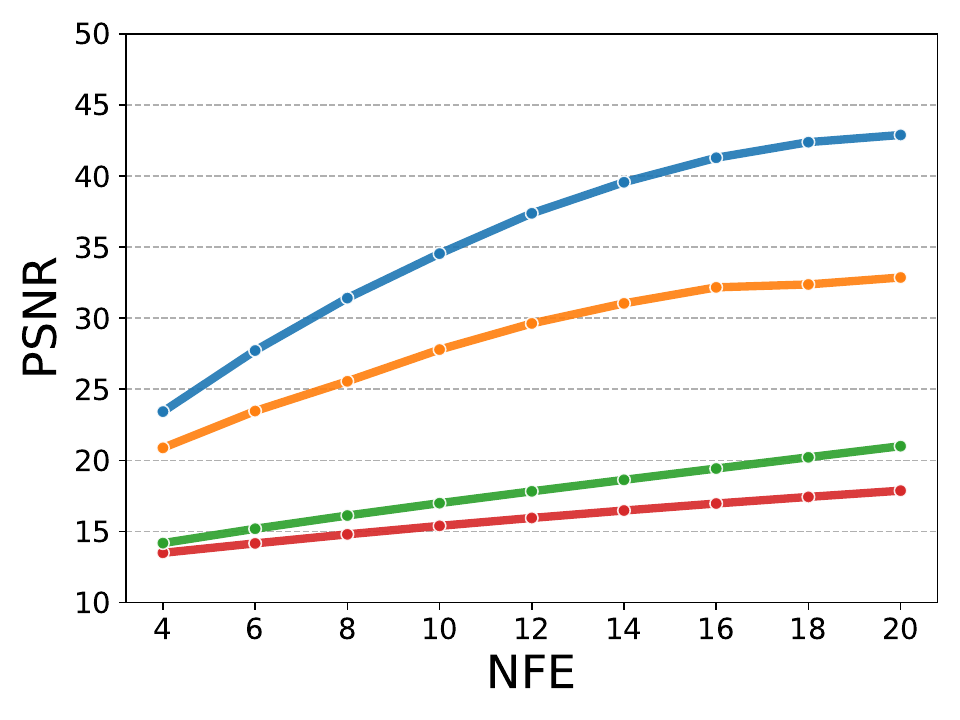}\\
    \includegraphics[width=0.25\textwidth]{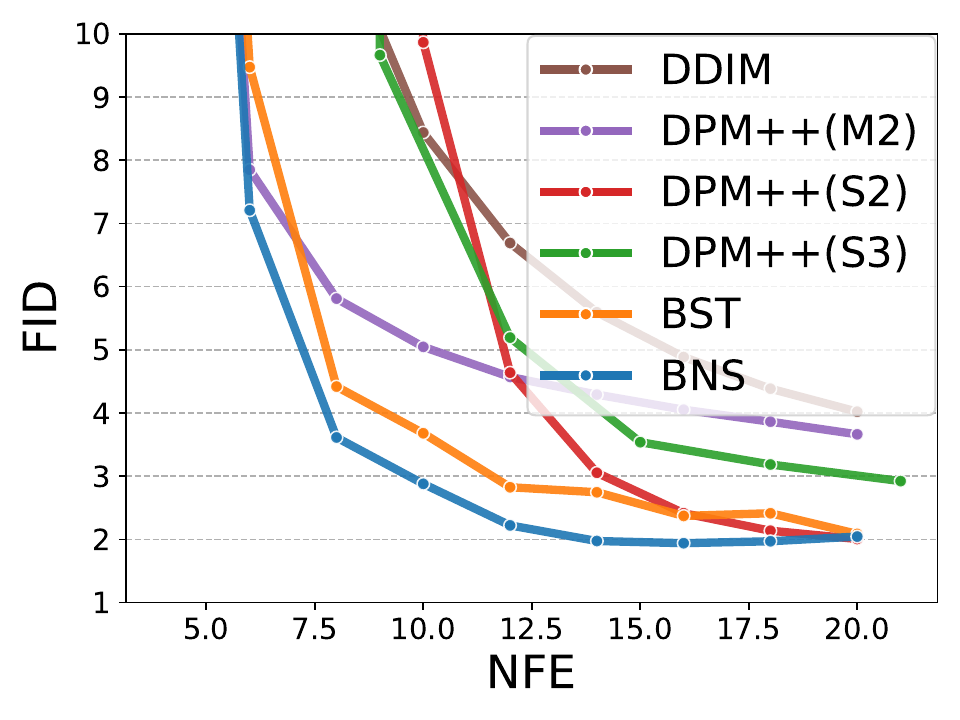} & \includegraphics[width=0.25\textwidth]{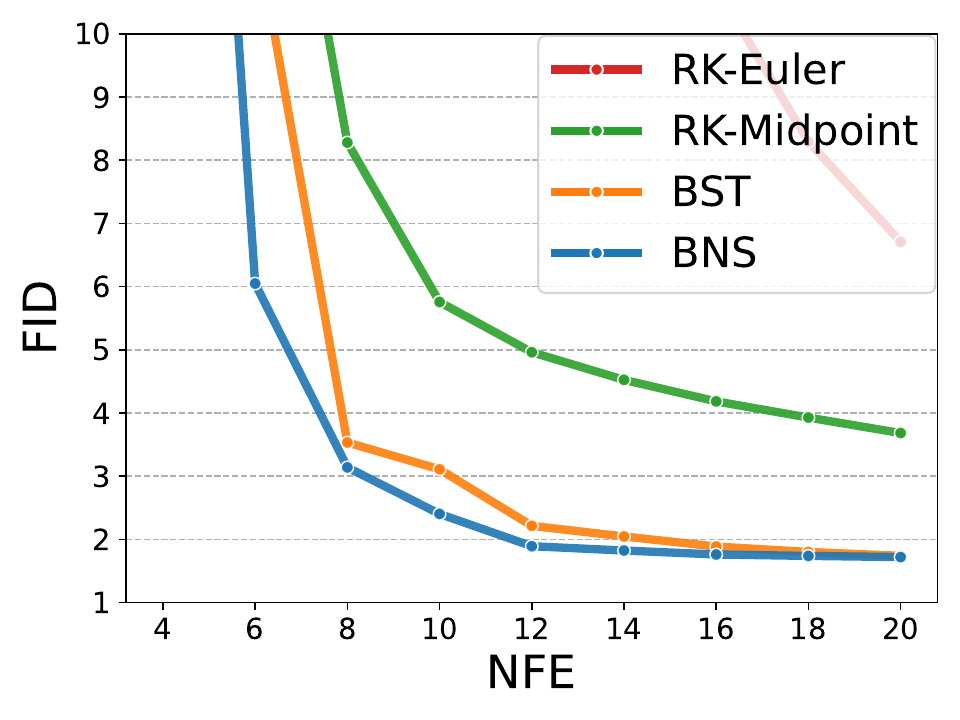} & \includegraphics[width=0.25\textwidth]{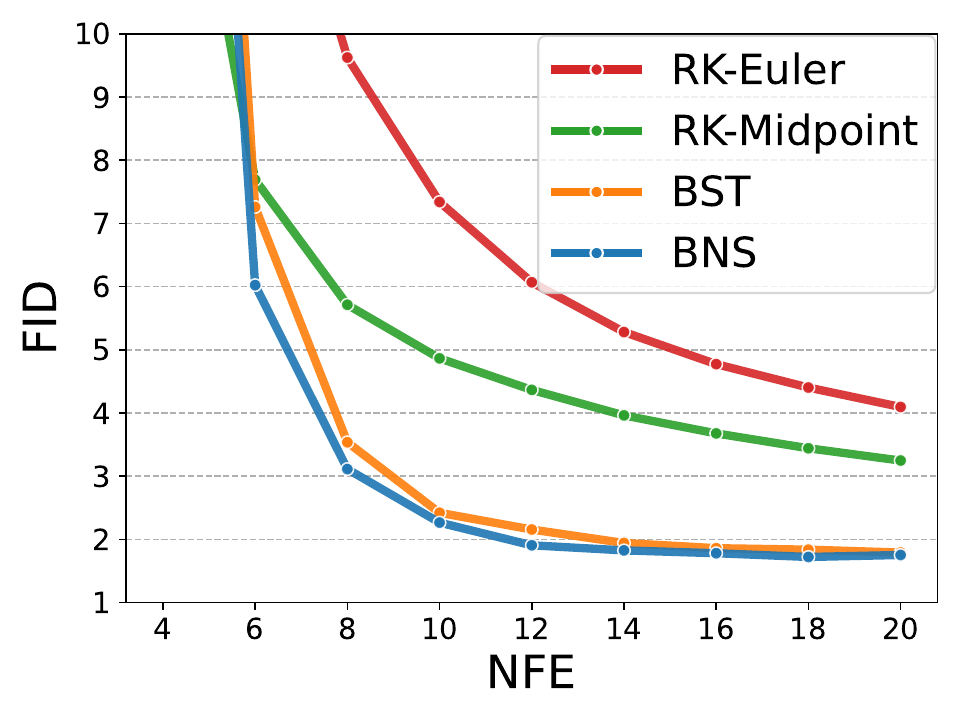} &
    \includegraphics[width=0.25\textwidth]{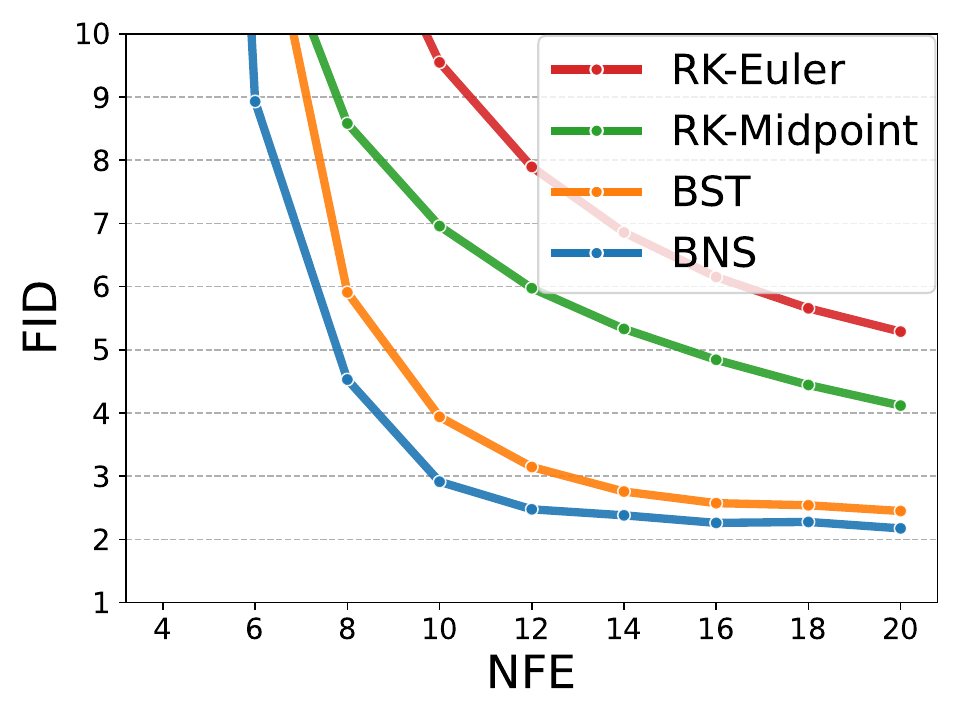}
\end{tabular}    
     \caption{
     BNS solvers vs.~BST solvers, RK-Midpoint/Euler, DDIM, DDIM, and DPM++ on ImageNet-64, and Image-Net128: PSNR vs.~NFE (top row), and FID vs.~NFE (bottom row).
     } 
    \label{fig:graph_imagenet_main}
\end{figure*}

%% file: tables/images_shutterstock.tex
\begin{figure*}[t]
\centering
\begin{tabular}{@{\hspace{2pt}}c@{\hspace{1pt}}c@{\hspace{1pt}}c@{\hspace{1pt}}c@{\hspace{2pt}}@{\hspace{2pt}}c@{\hspace{1pt}}c@{\hspace{1pt}}c@{\hspace{1pt}}c@{\hspace{1pt}}c@{\hspace{0pt}}} 
{\scriptsize GT} & {\scriptsize NFE=20} & {\scriptsize NFE=16} & {\scriptsize NFE=12} & {\scriptsize GT} & {\scriptsize NFE=20} & {\scriptsize NFE=16} & {\scriptsize NFE=12} \\ 
% \rotatebox[origin=c]{90}{\scriptsize FM-OT}
\cincludegraphics[width=0.120 \textwidth]{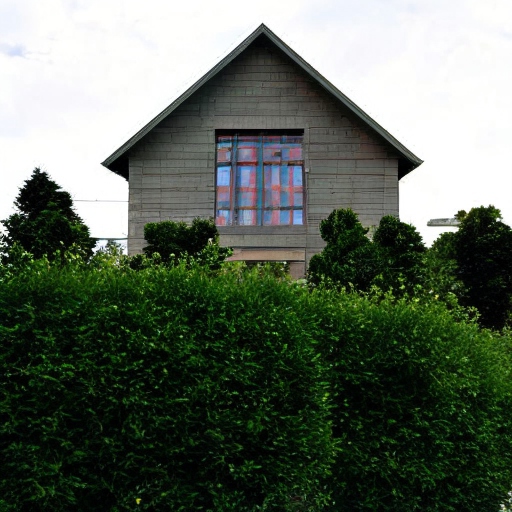} & 
    \begin{tabular}{@{\hspace{0pt}}c@{\hspace{0pt}}}    
        \includegraphics[width=0.120 \textwidth]{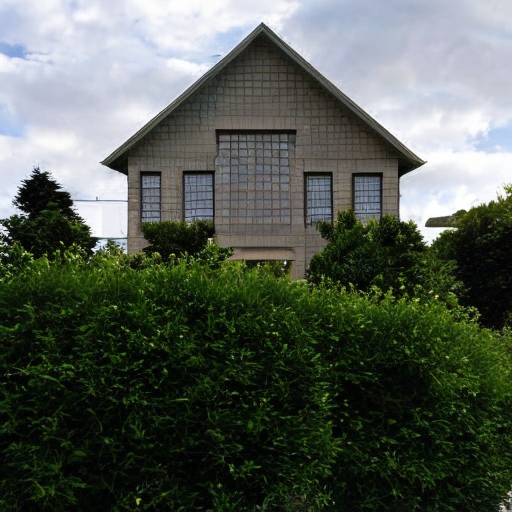} \\
        \includegraphics[width=0.120 \textwidth]{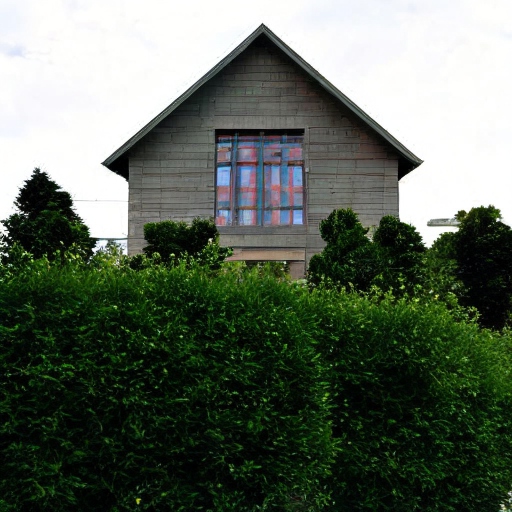}
    \end{tabular}
    &
    \begin{tabular}{@{\hspace{0pt}}c@{\hspace{0pt}}}   
        \includegraphics[width=0.120 \textwidth]{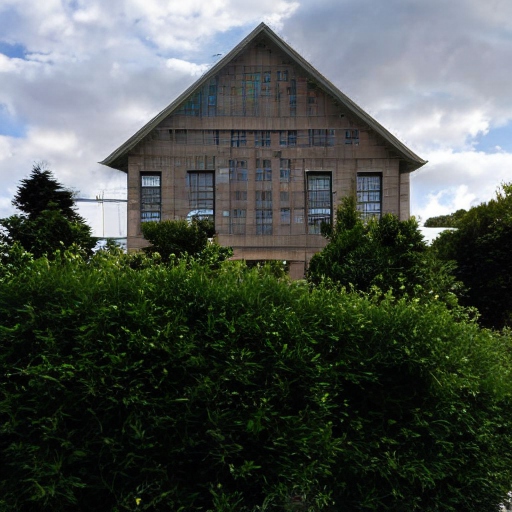} \\
        \includegraphics[width=0.120 \textwidth]{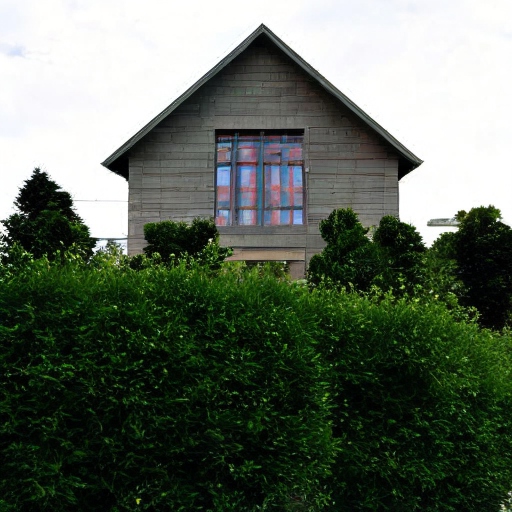}
    \end{tabular}
    &\begin{tabular}{@{\hspace{0pt}}c@{\hspace{0pt}}}  
        \includegraphics[width=0.120 \textwidth]{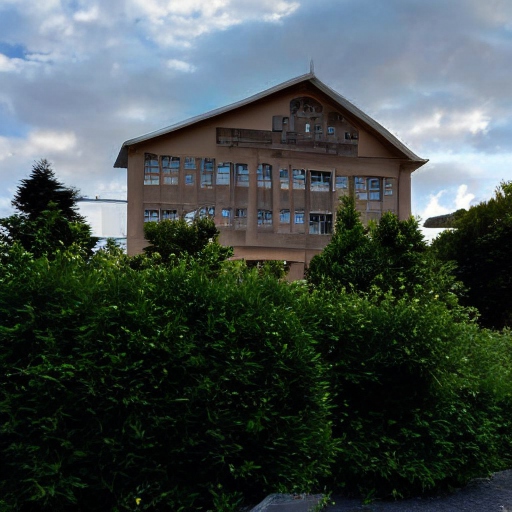} \\
        \includegraphics[width=0.120 \textwidth]{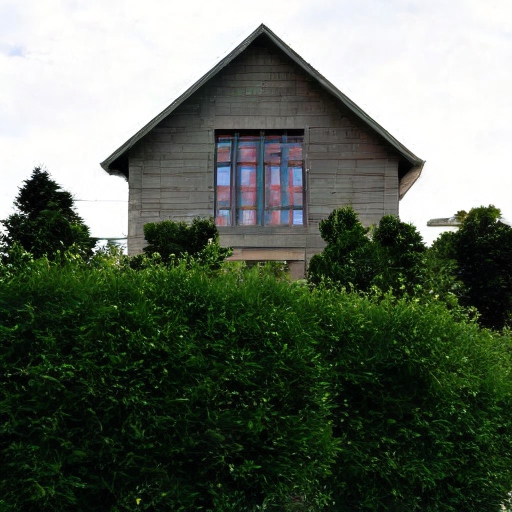}
    \end{tabular}
    &
\cincludegraphics[width=0.120 \textwidth]{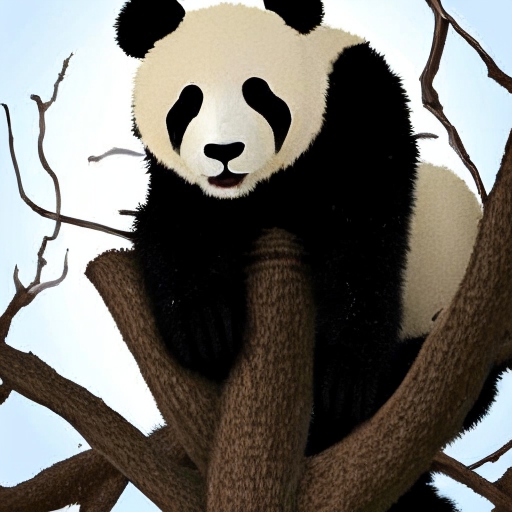} & 
    \begin{tabular}{@{\hspace{0pt}}c@{\hspace{0pt}}}   
        \includegraphics[width=0.120 \textwidth]{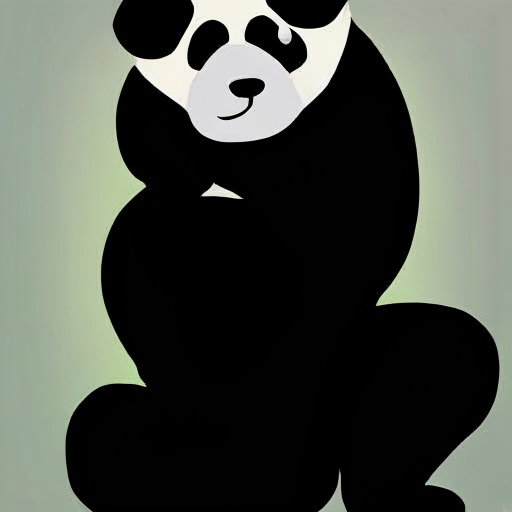} \\
        \includegraphics[width=0.120 \textwidth]{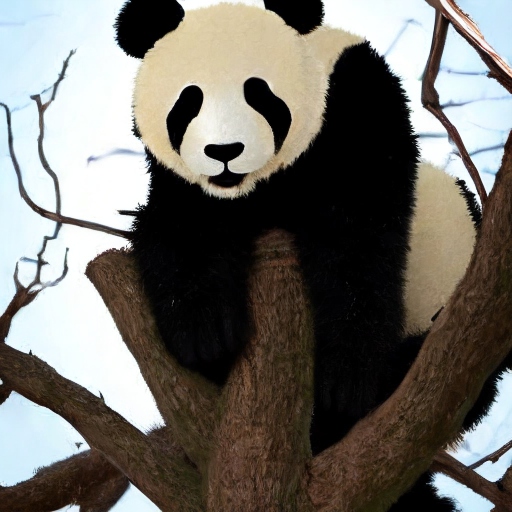}
    \end{tabular}
    &
    \begin{tabular}{@{\hspace{0pt}}c@{\hspace{0pt}}}
        \includegraphics[width=0.120 \textwidth]{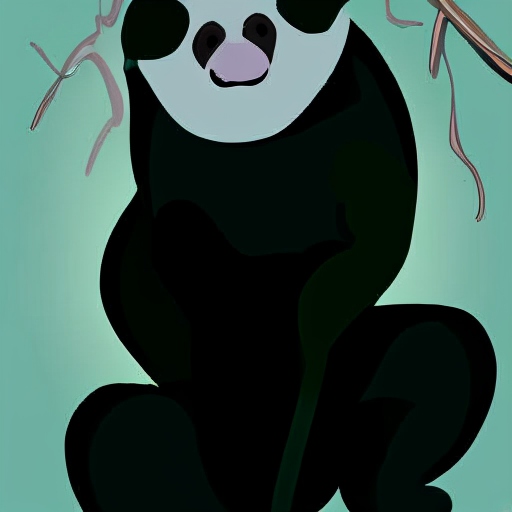} \\
        \includegraphics[width=0.120 \textwidth]{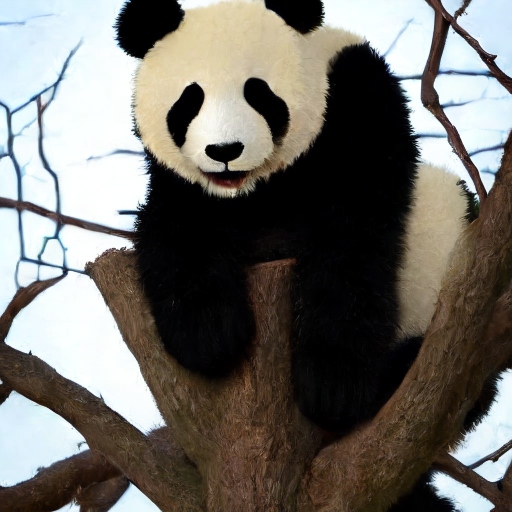}
    \end{tabular}
    &\begin{tabular}{@{\hspace{0pt}}c@{\hspace{0pt}}}   
        \includegraphics[width=0.120 \textwidth]{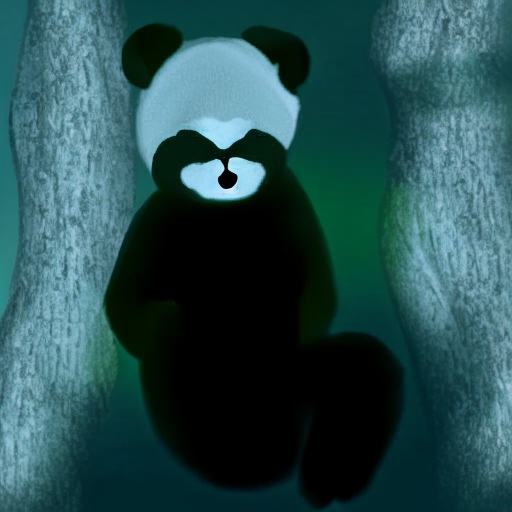} \\
        \includegraphics[width=0.120 \textwidth]{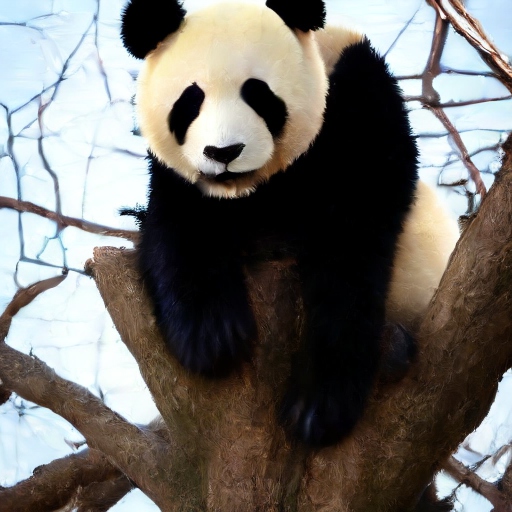}
    \end{tabular}
    &\begin{tabular}{@{\hspace{0pt}}c@{\hspace{0pt}}}   
        \raisebox{2.2\height}{\rotatebox[origin=c]{90}{\scriptsize RK-Midpoint}} \\
        \raisebox{2.0\height}{\rotatebox[origin=c]{90}{\scriptsize BNS}}
    \end{tabular} \vspace{-7pt}
\end{tabular}
\caption{
BNS vs.~RK-Midpoint on latent FM-OT Text-to-Image 512x512: (left) guidance scale $2.0$ with the caption \textit{"a building is shown behind trees and shrubs."}, (right) guidance scale $6.5$ with the \textit{"panda bear sitting in tree with no leaves."} \vspace{-10pt}
}\label{fig:images_shutterstock}
\end{figure*}

%% file: tables/table_shutterstock.tex
\begin{table}
\centering
\resizebox{\columnwidth}{!}{
\begin{tabular}{l@{\hspace{5pt}}c@{\hspace{5pt}}c@{\hspace{5pt}}c@{\hspace{5pt}}c@{\hspace{5pt}}c} 
 \toprule
  $w=2.0$ &  NFE & PSNR $\uparrow$ & Pick Score $\uparrow$ & Clip Score $\uparrow$ & FID $\downarrow$ \\ \midrule %\Xhline{1pt}
 \textit{GT (DOPRI5)}
 &
 170
 & 
 $\infty$
 &
 20.95
 &
 0.252
 &
 15.20\\[5pt]
% %
 % %
 % %
 \textit{RK-Euler}
 &
 \begin{tabular}[t]{c}
       12 \\ 16 \\ 20
 \end{tabular}
 & 
 \begin{tabular}[t]{c}
    13.95 \\ 14.86 \\ 15.71
 \end{tabular}
 &
 \begin{tabular}[t]{c}
    20.66 \\ 20.79 \\ 20.86
 \end{tabular}  
 &
 \begin{tabular}[t]{c}
    0.252 \\ 0.253 \\ 0.253
 \end{tabular} 
 &
 \begin{tabular}[t]{c}
    16.62 \\ 13.68 \\ 12.86
 \end{tabular} \\[30pt]
 % %
 % %
 % %
 \textit{RK-Midpoint}
 &
 \begin{tabular}[t]{c}
       12 \\ 16 \\ 20
 \end{tabular}
 & 
 \begin{tabular}[t]{c}
    15.05 \\ 16.28 \\ 17.46
 \end{tabular}
 &
 \begin{tabular}[t]{c}
    20.72 \\ 20.82 \\ 20.88
 \end{tabular}  
 &
 \begin{tabular}[t]{c}
    0.250 \\ 0.250 \\ 0.251
 \end{tabular} 
 &
 \begin{tabular}[t]{c}
    11.54 \\ 12.03 \\ 12.50
 \end{tabular} \\[30pt]
 % %
 % %
 % %
  \textbf{\textit{BNS}}
 &
 \begin{tabular}[t]{c}
       12 \\ 16 \\ 20
 \end{tabular}
 & 
 \begin{tabular}[t]{c}
 \cellcolor[HTML]{E8E8E8} 25.86 \\ \cellcolor[HTML]{E8E8E8} 29.13 \\ \cellcolor[HTML]{E8E8E8} 31.78
 \end{tabular}
 &
 \begin{tabular}[t]{c}
 \cellcolor[HTML]{E8E8E8} 20.83 \\ \cellcolor[HTML]{E8E8E8} 20.90 \\ \cellcolor[HTML]{E8E8E8} 20.91
 \end{tabular}  
 &
 \begin{tabular}[t]{c}
  0.252 \\ 0.252 \\ 0.252
 \end{tabular} 
 &
 \begin{tabular}[t]{c}
    13.93 \\ 14.48 \\ 14.68
 \end{tabular} \\
 \toprule
  $w=6.5$ &  NFE & PSNR $\uparrow$ & Pick Score $\uparrow$ & Clip Score $\uparrow$ & FID $\downarrow$ \\ \midrule %\Xhline{1pt}
 \textit{GT (DOPRI5)}
 &
 268
 & 
 $\infty$
 &
 21.16
 &
 0.260
 &
 23.99\\[5pt]
% %
 % %
 % %
 \textit{RK-Euler}
 &
 \begin{tabular}[t]{c}
       12 \\ 16 \\ 20
 \end{tabular}
 & 
 \begin{tabular}[t]{c}
    9.61 \\ 10.02 \\ 10.52
 \end{tabular}
 &
 \begin{tabular}[t]{c}
    19.92 \\ 20.34 \\ 20.60
 \end{tabular}  
 &
 \begin{tabular}[t]{c}
    0.237 \\ 0.247 \\ 0.252
 \end{tabular} 
 &
 \begin{tabular}[t]{c}
    50.00 \\ 35.37 \\ 28.36
 \end{tabular} \\[30pt]
 % %
 % %
 % %
 \textit{RK-Midpoint}
 &
 \begin{tabular}[t]{c}
       12 \\ 16 \\ 20
 \end{tabular}
 & 
 \begin{tabular}[t]{c}
    9.65 \\ 9.98 \\ 10.34
 \end{tabular}
 &
 \begin{tabular}[t]{c}
    19.79 \\ 20.11 \\ 20.34
 \end{tabular}  
 &
 \begin{tabular}[t]{c}
    0.240 \\ 0.245 \\ 0.248
 \end{tabular} 
 &
 \begin{tabular}[t]{c}
    34.01 \\ 27.06 \\ 23.63
 \end{tabular} \\[30pt]
 % %
 % %
 % %
  \textbf{\textit{BNS}}
 &
 \begin{tabular}[t]{c}
       12 \\ 16 \\ 20
 \end{tabular}
 & 
 \begin{tabular}[t]{c}
      \cellcolor[HTML]{E8E8E8} 18.94 \\ \cellcolor[HTML]{E8E8E8} 21.23 \\ \cellcolor[HTML]{E8E8E8} 23.27
 \end{tabular}
 &
 \begin{tabular}[t]{c}
    \cellcolor[HTML]{E8E8E8} 20.92 \\ \cellcolor[HTML]{E8E8E8} 21.03 \\ \cellcolor[HTML]{E8E8E8} 21.09
 \end{tabular}  
 &
 \begin{tabular}[t]{c}
       0.261 \\ 0.260 \\  0.259
 \end{tabular} 
 &
 \begin{tabular}[t]{c}
      20.67 \\  21.93 \\ 22.56
 \end{tabular} \\
\bottomrule  
\end{tabular}
}\vspace{-5pt}
\caption{
BNS solvers vs.~GT and RK-Midpoint, RK-Euler on Text-to-Image 512 FM-OT evaluated on MS-COCO.\vspace{-10pt}
}\label{tab:shutterstock}
\end{table}

%% file: tables/bespoke_vs_distillation.tex
% \begin{wraptable}[20]{r}{0.44\textwidth}\vspace{-10pt}
% \setlength{\tabcolsep}{2.0pt}
% \centering
% \resizebox{0.43\textwidth}{!}{
\begin{table}[]
    \centering
\resizebox{\columnwidth}{!}{
\begin{tabular}{l@{\hspace{0pt}}c@{\hspace{-1pt}}c@{\hspace{-1pt}}c@{\hspace{4pt}}c@{\hspace{0pt}}c@{\hspace{0pt}}c@{\hspace{0pt}}}
 %\Xhline{1pt} %\hline 
 \toprule
  CIFAR10 &  NFE & FID & GT-FID & Forwards & Training Set & Parameters  \\ \midrule %\Xhline{1pt}
 \textit{PD}%~\cite{salimans2022progressive} 
 \begin{tabular}[t]{c}
      \\ 
 \end{tabular}
 & 
 \begin{tabular}[t]{c}
      4\\ 8
 \end{tabular}
 &
 \begin{tabular}[t]{c}
      3.00 \\ 2.57 
 \end{tabular}  
 &
 \begin{tabular}[t]{c}
      2.51
 \end{tabular} 
 &
 \begin{tabular}[t]{c}
      211m\\192m
 \end{tabular}
 &
 \begin{tabular}[t]{c}
      50k \\ (CIFAR10)
 \end{tabular}
 &
 \begin{tabular}[t]{c}
      $>$50m
 \end{tabular} \\[20pt] %\Xhline{0.5pt}
\textbf{\textit{BNS}}
 & 
 \begin{tabular}[t]{c}
      4\\ 8\\ 
 \end{tabular}
 &
 \begin{tabular}[t]{c}
      25.20 \\ 2.73 
 \end{tabular}  
 &
 \begin{tabular}[t]{c}
      2.54
 \end{tabular} 
 &
 \begin{tabular}[t]{c}
      4.9m \\ 9.7m
 \end{tabular}
 &
 \begin{tabular}[t]{c}
      520
 \end{tabular}
 &
 \begin{tabular}[t]{c}
      18 \\ 52
 \end{tabular}\\ %[45pt] %\Xhline{0.5pt}
 %
 %
 %
  %\Xhline{1pt} %\hline 
  \toprule
  ImageNet-64 &  NFE & FID & GT-FID & Forwards & Training Set & Parameters  \\ \midrule %\Xhline{1pt}
 \textit{PD}%~\cite{meng2023distillation} 
 & 
 \begin{tabular}[t]{c}
      4\\ 8\\ 16
 \end{tabular}
 &
 \begin{tabular}[t]{c}
    4.79 \\ 3.39\\ 2.97 
 \end{tabular}  
 &
 \begin{tabular}[t]{c}
      2.92
 \end{tabular} 
 &
 \begin{tabular}[t]{c}
      2457m \\ 2150m \\ 1843m
 \end{tabular}
 &
 \begin{tabular}[t]{c}
      1.2m \\ (ImageNet)
 \end{tabular}
 &
 \begin{tabular}[t]{c}
      $>$ 200m 
 \end{tabular} \\[30pt] %\Xhline{0.5pt}
 %
 %
 %
 % \textit{PD}~\cite{meng2023distillation} 
 % & 
 % \begin{tabular}[t]{c}
 %      4\\ 8\\ 12\\ 16
 % \end{tabular}
 % &
 % \begin{tabular}[t]{c}
 %    4.14 \\ 2.79\\ - \\ 2.44 
 % \end{tabular}  
 % &
 % \begin{tabular}[t]{c}
 %      2.92
 % \end{tabular} 
 % &
 % \begin{tabular}[t]{c}
 %      XXX
 % \end{tabular}  \\[45pt] %\Xhline{0.5pt}
 %
 %
 %
\textbf{\textit{BNS}}
 & 
 \begin{tabular}[t]{c}
      4\\ 8\\ 16
 \end{tabular}
 &
 \begin{tabular}[t]{c}
      31.83 \\ 3.90 \\ 2.62 
 \end{tabular}  
 &
 \begin{tabular}[t]{c}
      2.50
 \end{tabular} 
 &
 \begin{tabular}[t]{c}
      2.5m \\ 4.9m \\ 9.7m 
 \end{tabular}
 &
 \begin{tabular}[t]{c}
      520
 \end{tabular} 
 &
 \begin{tabular}[t]{c}
      18 \\ 52 \\ 168 
 \end{tabular}  
 \\ 
 %\Xhline{1pt} 
 \bottomrule  
\end{tabular} 
}
    \caption{
    BNS solver vs.~Progressive Distillation on CIFAR10 and ImageNet-64$^1$ class conditional with $w=0$. \vspace{-12pt}
    }
\label{tab:bespoke_vs_disitillation}
\end{table}

% } \vspace{-5pt}

% \end{wraptable} 

%% file: tables/graphs_audio.tex
\begin{figure}[t]
    \centering
    \begin{tabular}{@{\hspace{0pt}}c@{\hspace{0pt}}c@{\hspace{0pt}}}
    {\quad \ \scriptsize LibriSpeech TTS} & {\quad \ \scriptsize Audiocaps}\\
    \includegraphics[width=0.23\textwidth]{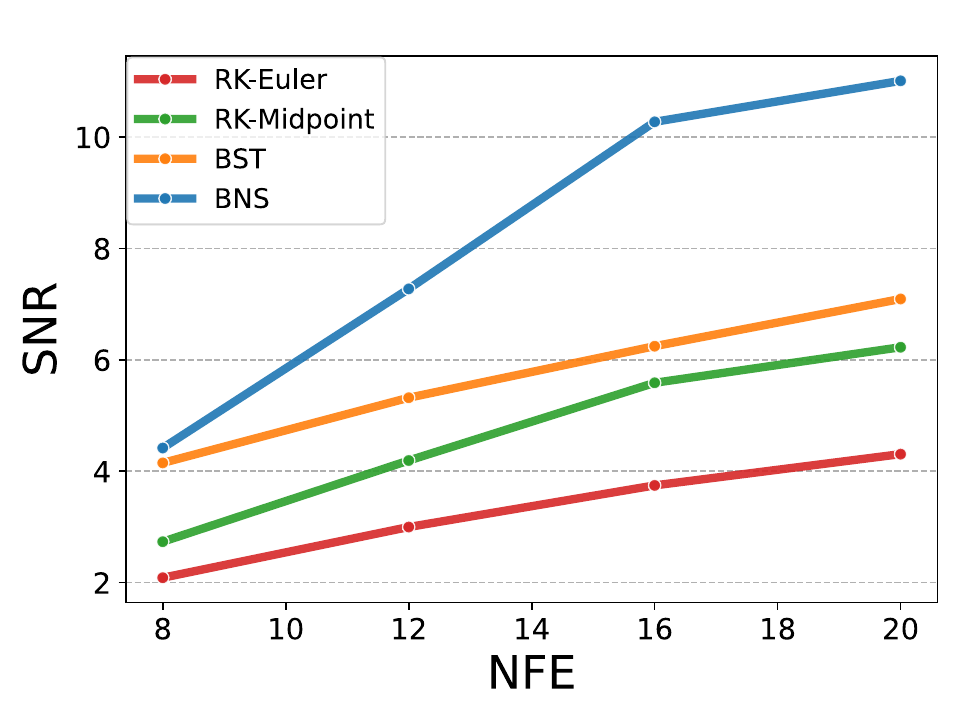}  & \includegraphics[width=0.23\textwidth]{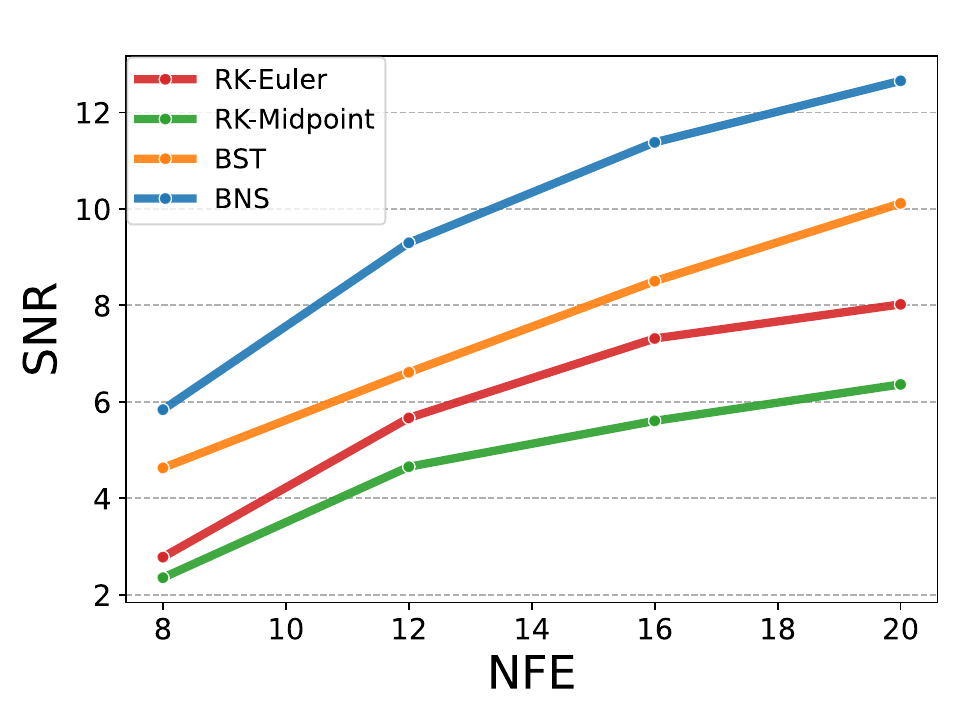} 
\end{tabular}    
     \caption{NFE vs.~SNR of BNS solvers, BST solvers, RK-Midpoint/Euler for Speech Generation FM-OT evaluated on: (left) LibriSpeech TTS, (right) Audiocaps. \vspace{-5pt} } 
    \label{fig:graphs_t2a}
\end{figure}

%% file: tables/images_shutterstock_w2.0_appendix.tex
\begin{figure*}[!htbp]
    \centering

    \caption{
    Comparison of generated Images on latent FM-OT text-to-image 512x512 guidance scale $2.0$: RK-Midpoint, Initial Solver (RK-Euler+precondition $\sigma_0=5$), and BNS. 
    }\label{fig:a_images_t2i_w2.0}
    \end{figure*}

%% file: tables/images_shutterstock_w6.5_appendix.tex
\begin{figure*}[!htbp]
    \centering
    %
    \caption{
    Comparison of generated Images on latent FM-OT text-to-image 512x512 guidance scale $6.5$: RK-Midpoint, Initial Solver (RK-Euler+precondition $\sigma_0=10$), and BNS.
    }\label{fig:a_images_t2i_w6.5}
    \end{figure*}

%% file: tables/images_imagenet128.tex
\begin{figure*}[!htbp]
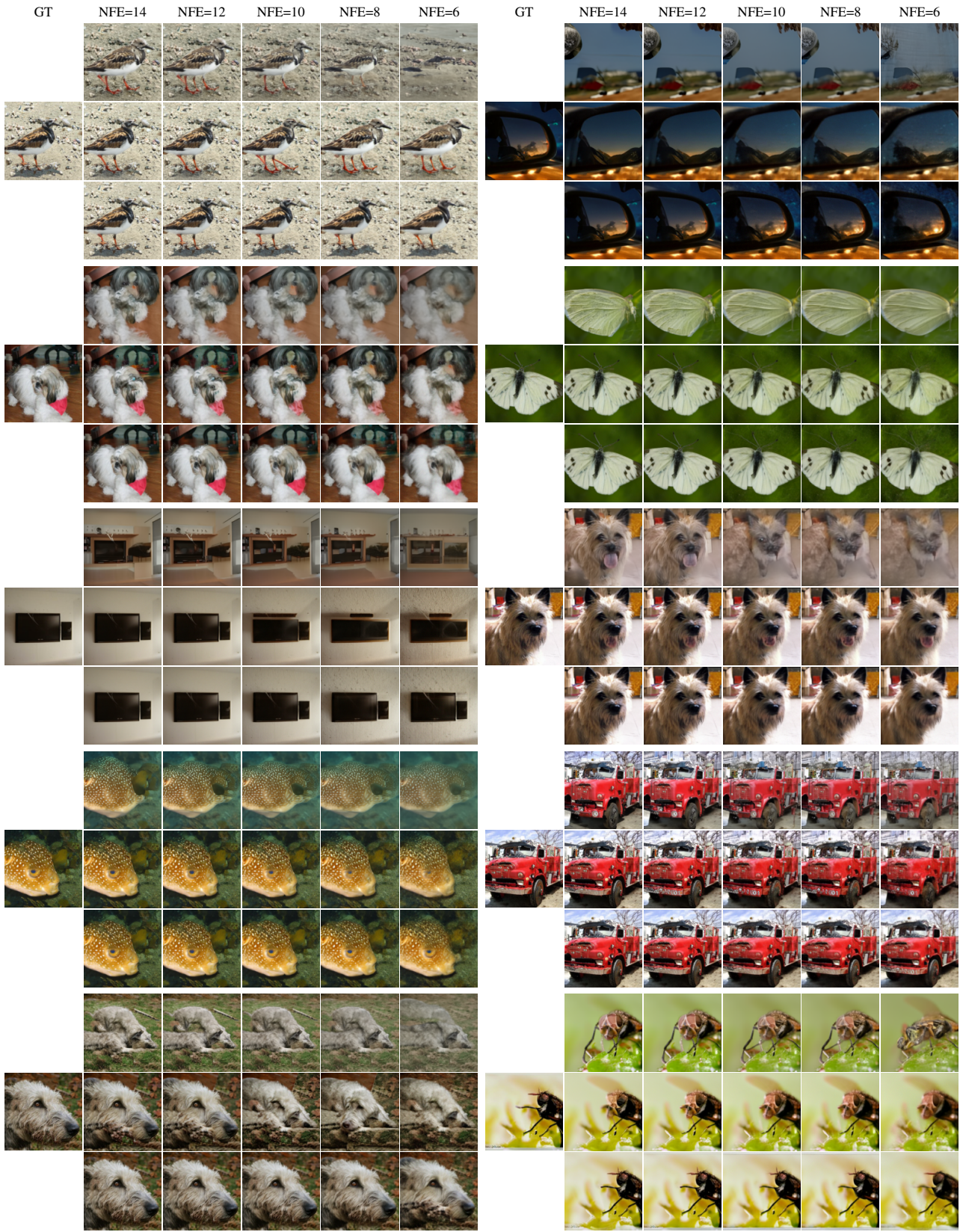

    \centering
    %
    \caption{
    Comparison of generated Images on ImageNet-128 FM-OT model with guidance scale $0.5$: (top row) RK-Midpoint, (middle row) BST, (bottom row) BNS. 
    }\label{fig:a_images_imagenet128}
    \end{figure*}

%% file: tables/images_imagenet64_eps_vp.tex
\begin{figure*}[!htbp]
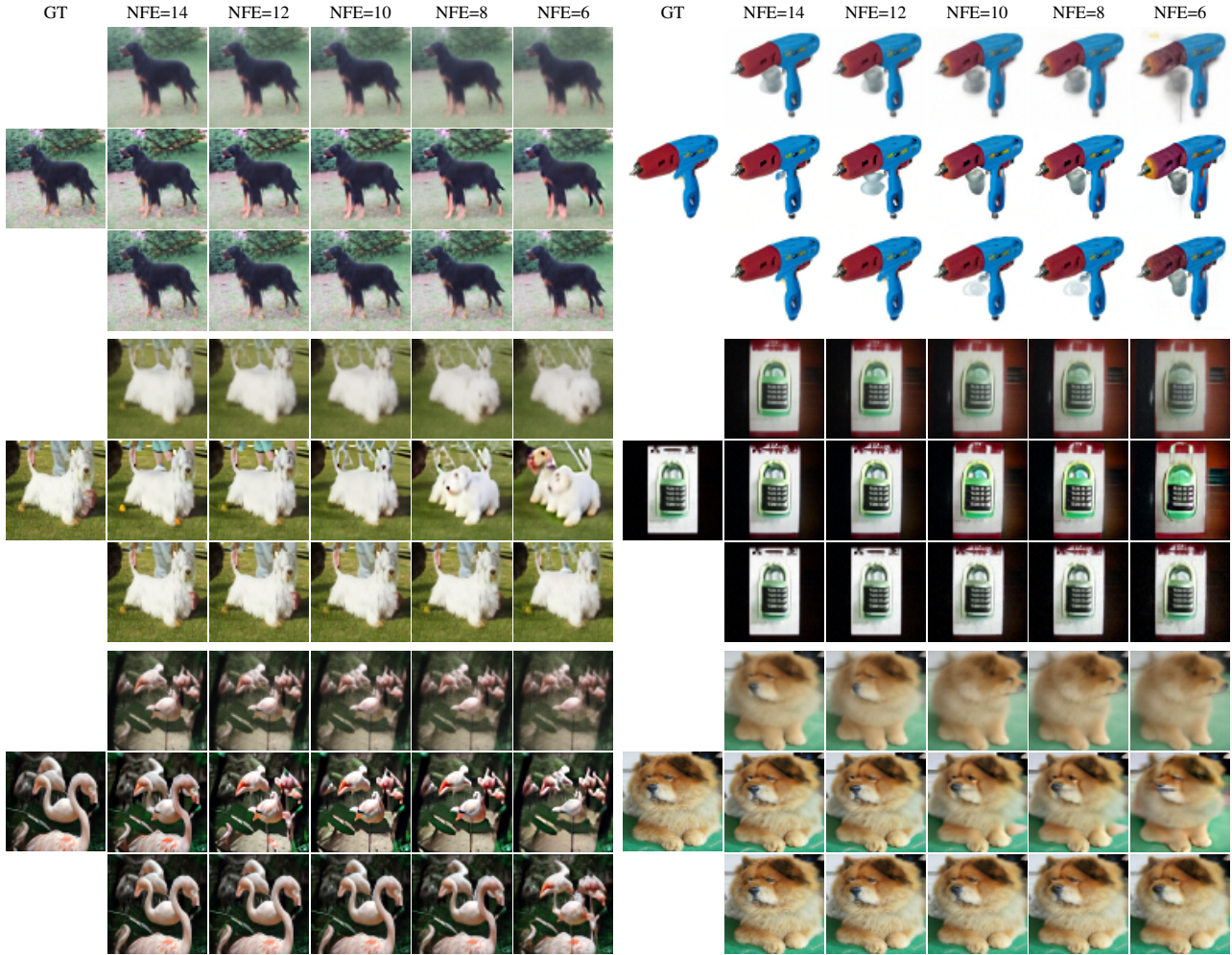

    \centering
    %
\\
    \end{tabular}
    \caption{
    Comparison of generated Images on ImageNet-64 $\eps$-VP model with guidance scale $0.2$: (top row) DDIM, (middle row) DPM++(2M), (bottom row) BNS. 
    }\label{fig:a_images_imagenet64_eps_vp}
    \end{figure*}

%% file: tables/graphs_family_ablation.tex
\begin{figure}[t]
    \centering
    \begin{tabular}{@{\hspace{0pt}}c@{\hspace{0pt}}c@{\hspace{0pt}}}
    \multicolumn{2}{c}{{\quad \ \scriptsize ImageNet-64: FM-OT}}\\
       
    \includegraphics[width=0.25\textwidth]{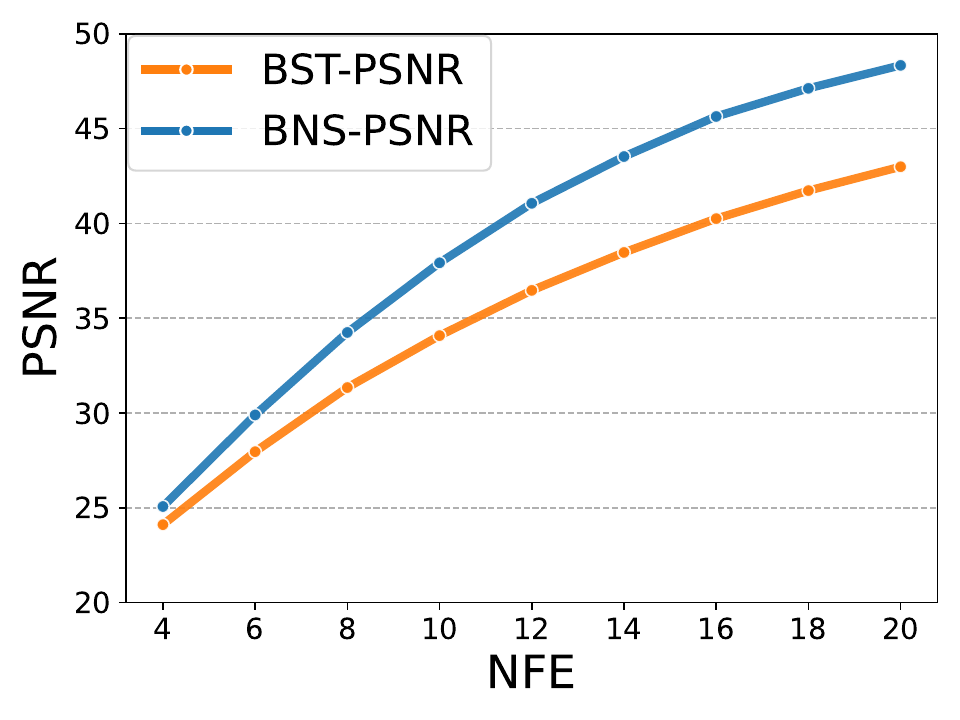}  & \includegraphics[width=0.25\textwidth]{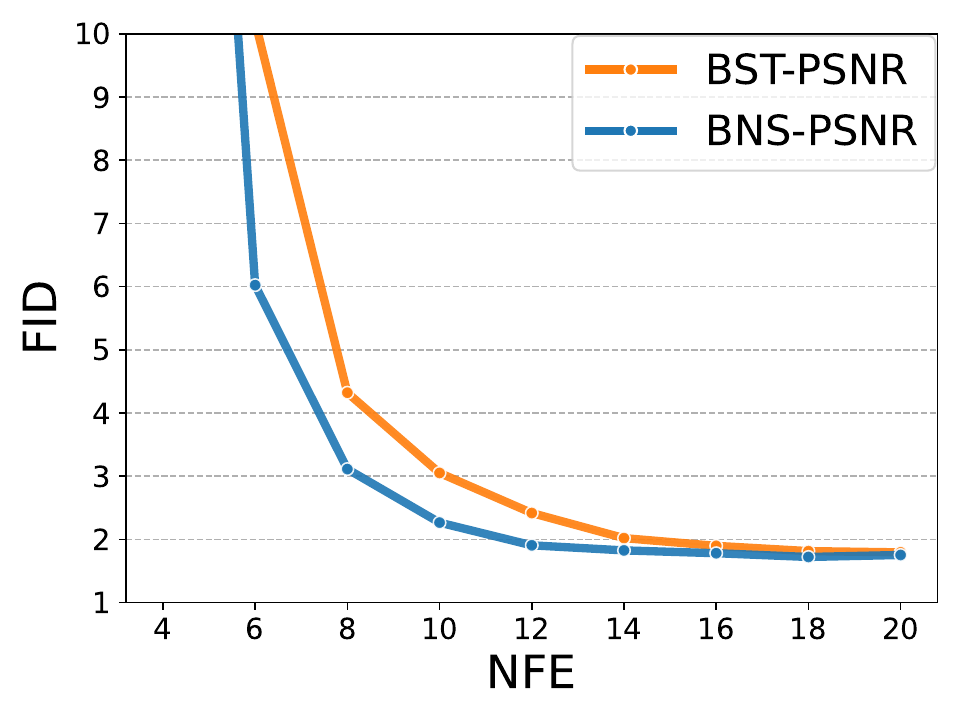} 
\end{tabular}    
     \caption{BNS vs.~BST on ImageNet-64 FM-OT trained with PSNR loss. } 
    \label{fig:graph_imagenet_BNS_vs_BST}
\end{figure}

%% file: tables/table_imagenet_appendix.tex
\begin{table}
\centering
\resizebox{0.98\textwidth}{!}{
\begin{tabular}{llcccccccccc} 
 %\Xhline{1pt} %\hline 
 \toprule
  ImageNet 64 & \begin{tabular}[t]{l}NFE:\end{tabular} & 4 & 6 & 8 & 10 & 12 & 14 & 16 &18 & 20 & GT \\ \midrule %\Xhline{1pt}
 \textbf{\textit{FM-OT}}
 &
 \begin{tabular}[t]{l}
     PSNR: \\ FID:
 \end{tabular}
 &
 \begin{tabular}[t]{c}
     25.08 \\ 27.35
 \end{tabular}
 &
 \begin{tabular}[t]{c}
     29.9 \\ 6.02
 \end{tabular}
 &
 \begin{tabular}[t]{c}
     34.25 \\ 3.11
 \end{tabular}
 &
 \begin{tabular}[t]{c}
     37.92 \\ 2.27
 \end{tabular}
 &
 \begin{tabular}[t]{c}
     41.06 \\ 1.91
 \end{tabular}
 &
 \begin{tabular}[t]{c}
     43.52\ \\ 1.83
 \end{tabular}
 &
 \begin{tabular}[t]{c}
     45.64 \\ 1.78
 \end{tabular}
 &
 \begin{tabular}[t]{c}
     47.12 \\ 1.72
 \end{tabular}
 &
 \begin{tabular}[t]{c}
     48.33 \\ 1.75
 \end{tabular}
 &
 \begin{tabular}[t]{c}
     $\infty$ \\ 1.68
 \end{tabular}\\[20pt]
 % %
 % %
 % %
 \textbf{\textit{FM$\bm{v}$-CS}}
 &
 \begin{tabular}[t]{l}
     PSNR: \\ FID:
 \end{tabular}
 &
 \begin{tabular}[t]{c}
     25.0 \\ 27.59
 \end{tabular}
 &
 \begin{tabular}[t]{c}
     29.76 \\ 6.05
 \end{tabular}
 &
 \begin{tabular}[t]{c}
     34.03 \\ 3.14
 \end{tabular}
 &
 \begin{tabular}[t]{c}
     37.7 \\ 2.4
 \end{tabular}
 &
 \begin{tabular}[t]{c}
     40.83 \\ 1.89
 \end{tabular}
 &
 \begin{tabular}[t]{c}
     43.3 \\ 1.82
 \end{tabular}
 &
 \begin{tabular}[t]{c}
     45.21 \\ 1.76
 \end{tabular}
 &
 \begin{tabular}[t]{c}
     46.68 \\ 1.74
 \end{tabular}
 &
 \begin{tabular}[t]{c}
     47.68 \\ 1.72
 \end{tabular}
 &
 \begin{tabular}[t]{c}
     $\infty$ \\ 1.71
 \end{tabular}\\[20pt]
 % %
 % %
 % %
 \textbf{\textit{$\bm{\eps}$-VP}}
 &
 \begin{tabular}[t]{l}
     PSNR: \\ FID:
 \end{tabular}
 &
 \begin{tabular}[t]{c}
     24.65 \\ 30.0
 \end{tabular}
 &
 \begin{tabular}[t]{c}
     29.49 \\ 7.21
 \end{tabular}
 &
 \begin{tabular}[t]{c}
     33.77 \\ 3.61
 \end{tabular}
 &
 \begin{tabular}[t]{c}
     37.48 \\ 2.88
 \end{tabular}
 &
 \begin{tabular}[t]{c}
     40.61 \\ 2.22
 \end{tabular}
 &
 \begin{tabular}[t]{c}
     43.21 \\ 1.97
 \end{tabular}
 &
 \begin{tabular}[t]{c}
     45.7 \\ 1.94
 \end{tabular}
 &
 \begin{tabular}[t]{c}
     47.32 \\ 1.97
 \end{tabular}
 &
 \begin{tabular}[t]{c}
     48.57 \\ 2.04
 \end{tabular}
 &
 \begin{tabular}[t]{c}
     $\infty$ \\ 1.84
 \end{tabular}\\ \midrule%[20pt]
 % %
 % %
 % %
 ImageNet 128 & \begin{tabular}[t]{l}NFE:\end{tabular} & 4 & 6 & 8 & 10 & 12 & 14 & 16 & 18 & 20 & GT \\ \midrule %\Xhline{1pt}
 \textbf{\textit{FM-OT}}
 &
 \begin{tabular}[t]{l}
     PSNR: \\ FID:
 \end{tabular}
 &
 \begin{tabular}[t]{c}
     23.43 \\ 36.17
 \end{tabular}
 &
 \begin{tabular}[t]{c}
     27.72 \\ 8.93
 \end{tabular}
 &
 \begin{tabular}[t]{c}
     31.41 \\ 4.53
 \end{tabular}
 &
 \begin{tabular}[t]{c}
     34.54 \\ 2.91
 \end{tabular}
 &
 \begin{tabular}[t]{c}
     37.37 \\ 2.48
 \end{tabular}
 &
 \begin{tabular}[t]{c}
     39.56 \\  2.38
 \end{tabular}
 &
 \begin{tabular}[t]{c}
     41.28 \\ 2.26
 \end{tabular}
 &
 \begin{tabular}[t]{c}
     42.38 \\ 2.28
 \end{tabular}
 &
 \begin{tabular}[t]{c}
     42.88 \\ 2.17
 \end{tabular}
 &
 \begin{tabular}[t]{c}
     $\infty$ \\ 2.16
 \end{tabular}\\%[20pt]
\bottomrule  
\end{tabular}
}\vspace{-5pt}
\caption{
PSNR and FID of BNS on ImageNet 64 FM-OT/FM$v$-CS/$\eps$-VP, and ImageNet 128 FM-OT.
}\label{tab:imagenet}
\end{table}

%% file: tables/table_shutterstock_appendix.tex
\begin{table}
\centering
\resizebox{0.5\textwidth}{!}{
\begin{tabular}{lccccc} 
 %\Xhline{1pt} %\hline 
 \toprule
  $w=2.0$ &  NFE & PSNR & Pick Score & Clip Score & FID \\ \midrule %\Xhline{1pt}
 \textit{GT (DOPRI5)}
 &
 170
 & 
 $\infty$
 &
 20.95
 &
 0.252
 &
 15.20\\[5pt]
% %
 % %
 % %
 \textbf{\textit{Initial Solver}}
 &
 \begin{tabular}[t]{c}
       12 \\ 16 \\ 20
 \end{tabular}
 & 
 \begin{tabular}[t]{c}
    19.23 \\ 20.55 \\ 21.60
 \end{tabular}
 &
 \begin{tabular}[t]{c}
    20.65 \\ 20.78 \\ 20.85
 \end{tabular}  
 &
 \begin{tabular}[t]{c}
    0.257 \\ 0.256 \\ 0.256
 \end{tabular} 
 &
 \begin{tabular}[t]{c}
    21.14 \\ 18.19 \\ 16.96
 \end{tabular} \\[30pt]
 % %
 % %
 % %
  \textbf{\textit{BNS}}
 &
 \begin{tabular}[t]{c}
       12 \\ 16 \\ 20
 \end{tabular}
 & 
 \begin{tabular}[t]{c}
 \cellcolor[HTML]{EFEFEF} 25.86 \\ \cellcolor[HTML]{EFEFEF} 29.13 \\ \cellcolor[HTML]{EFEFEF} 31.78
 \end{tabular}
 &
 \begin{tabular}[t]{c}
 \cellcolor[HTML]{EFEFEF} 20.83 \\ \cellcolor[HTML]{EFEFEF} 20.90 \\ \cellcolor[HTML]{EFEFEF} 20.91
 \end{tabular}  
 &
 \begin{tabular}[t]{c}
  0.252 \\ 0.252 \\ 0.252
 \end{tabular} 
 &
 \begin{tabular}[t]{c}
    13.93 \\ 14.48 \\ 14.68
 \end{tabular} \\
 \toprule
  $w=6.5$ &  NFE & PSNR & Pick Score & Clip Score & FID \\ \midrule %\Xhline{1pt}
 \textit{GT (DOPRI5)}
 &
 268
 & 
 $\infty$
 &
 21.16
 &
 0.260
 &
 23.99\\[5pt]
 % %
 % %
 % %
 \textbf{\textit{Initial Solver}}
 &
 \begin{tabular}[t]{c}
       12 \\ 16 \\ 20
 \end{tabular}
 & 
 \begin{tabular}[t]{c}
 17.21 \\ 18.38 \\ 19.29
 \end{tabular}
 &
 \begin{tabular}[t]{c}
    20.69 \\ 20.87 \\ 20.97
 \end{tabular}  
 &
 \begin{tabular}[t]{c}
    0.264 \\ 0.263 \\ 0.262
 \end{tabular} 
 &
 \begin{tabular}[t]{c}
    29.63 \\ 28.02 \\ 27.15
 \end{tabular} \\[30pt]
 % %
 % %
 % %
  \textbf{\textit{BNS}}
 &
 \begin{tabular}[t]{c}
       12 \\ 16 \\ 20
 \end{tabular}
 & 
 \begin{tabular}[t]{c}
      \cellcolor[HTML]{EFEFEF} 18.94 \\ \cellcolor[HTML]{EFEFEF} 21.23 \\ \cellcolor[HTML]{EFEFEF} 23.27
 \end{tabular}
 &
 \begin{tabular}[t]{c}
    \cellcolor[HTML]{EFEFEF} 20.92 \\ \cellcolor[HTML]{EFEFEF} 21.03 \\ \cellcolor[HTML]{EFEFEF} 21.09
 \end{tabular}  
 &
 \begin{tabular}[t]{c}
      0.261 \\  0.260 \\ 0.259
 \end{tabular} 
 &
 \begin{tabular}[t]{c}
      20.67 \\ 21.93 \\  22.56
 \end{tabular} \\
\bottomrule  
\end{tabular}
}\vspace{-5pt}
\caption{
BNS solvers vs.~GT and Intial Solver (Euler + ST) on Text-to-Image 512 FM-OT evaluated on MS-COCO.
}\label{tab:shutterstock_initial_solver}
\end{table} 

%% file: tables/graphs_audio_appendix.tex
\begin{figure*}[t]
\centering
\begin{tabular}{@{\hspace{0pt}}c@{\hspace{0pt}}c@{\hspace{0pt}}c@{\hspace{0pt}}c@{\hspace{0pt}}} 
    \begin{tabular}{@{\hspace{0pt}}c@{\hspace{0pt}}}
        {\quad \ \scriptsize Spotify}\\
        \includegraphics[width=0.25 \textwidth]{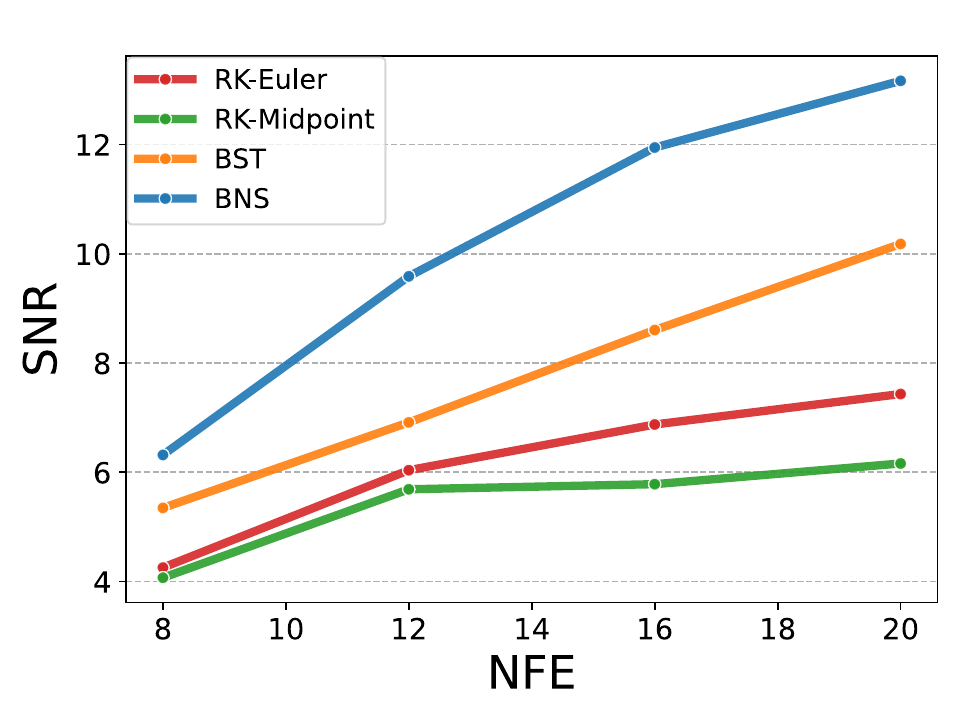} \\
        {\quad \ \scriptsize Accent}\\
        \includegraphics[width=0.25 \textwidth]{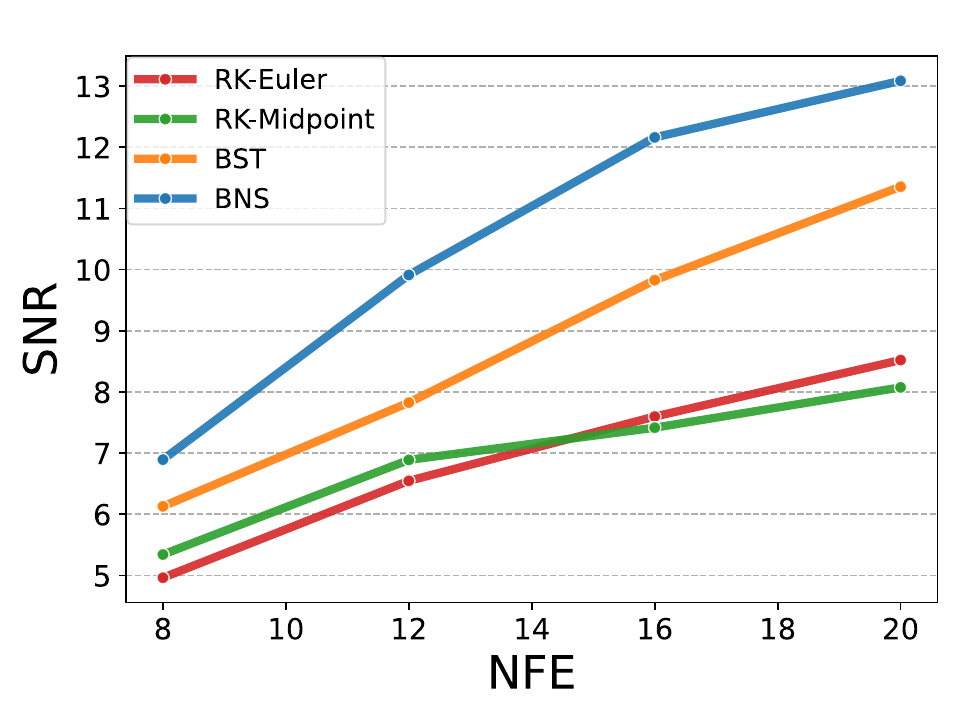}
    \end{tabular}
    &
    \begin{tabular}{@{\hspace{0pt}}c@{\hspace{0pt}}}
        {\quad \ \scriptsize Expresso}\\
        \includegraphics[width=0.25 \textwidth]{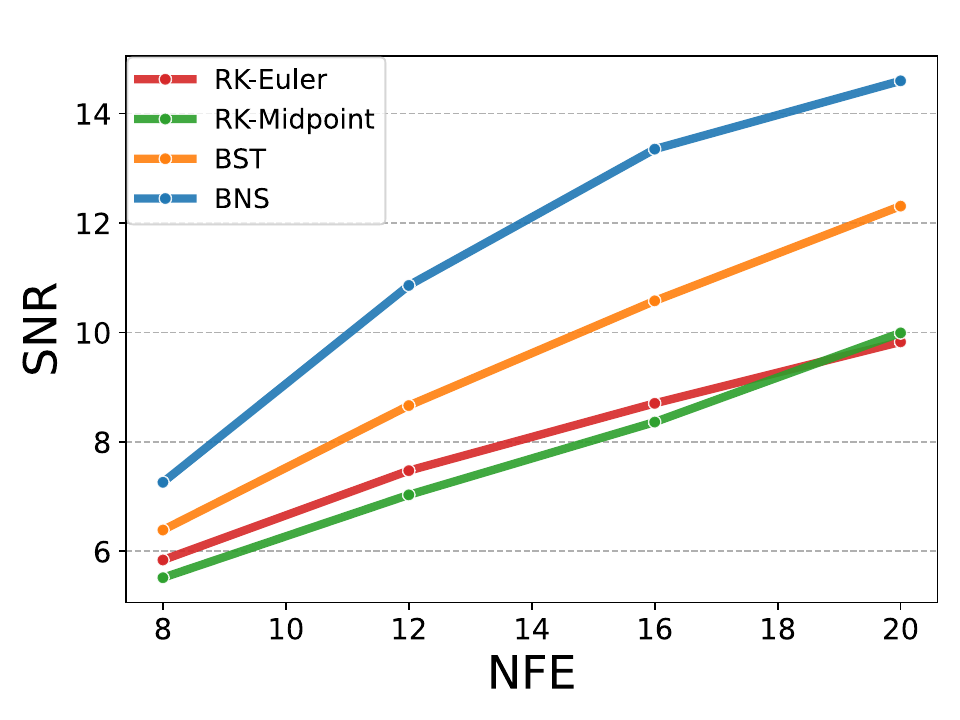} \\
        {\quad \ \scriptsize Switchboard}\\
        \includegraphics[width=0.25 \textwidth]{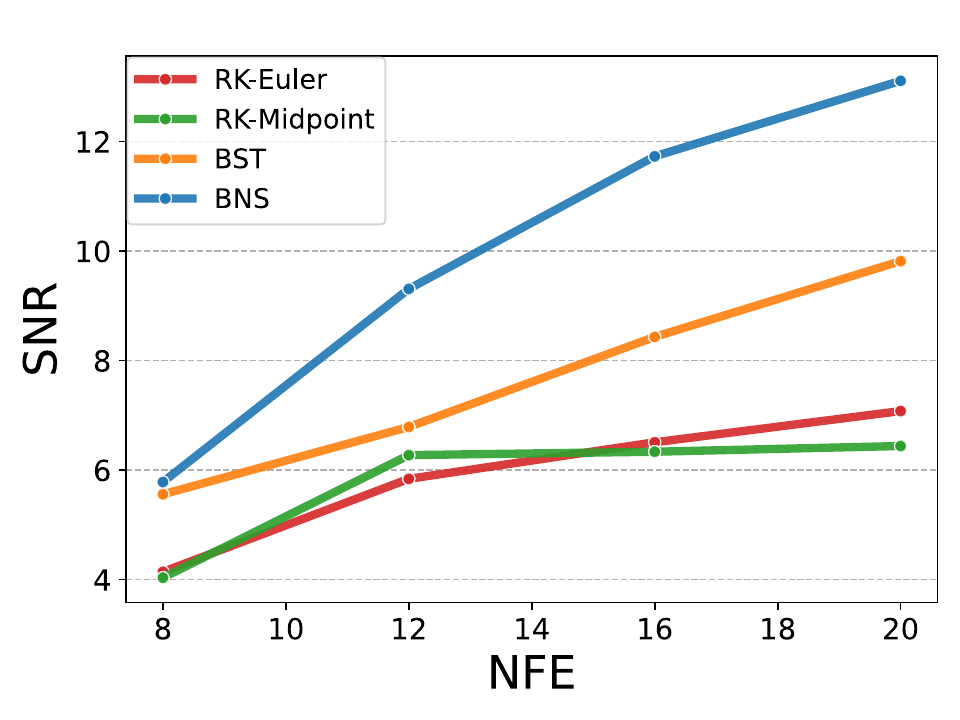}
    \end{tabular}
    &\begin{tabular}{@{\hspace{0pt}}c@{\hspace{0pt}}}  
        {\quad \ \scriptsize Librispeech}\\
        \includegraphics[width=0.25 \textwidth]{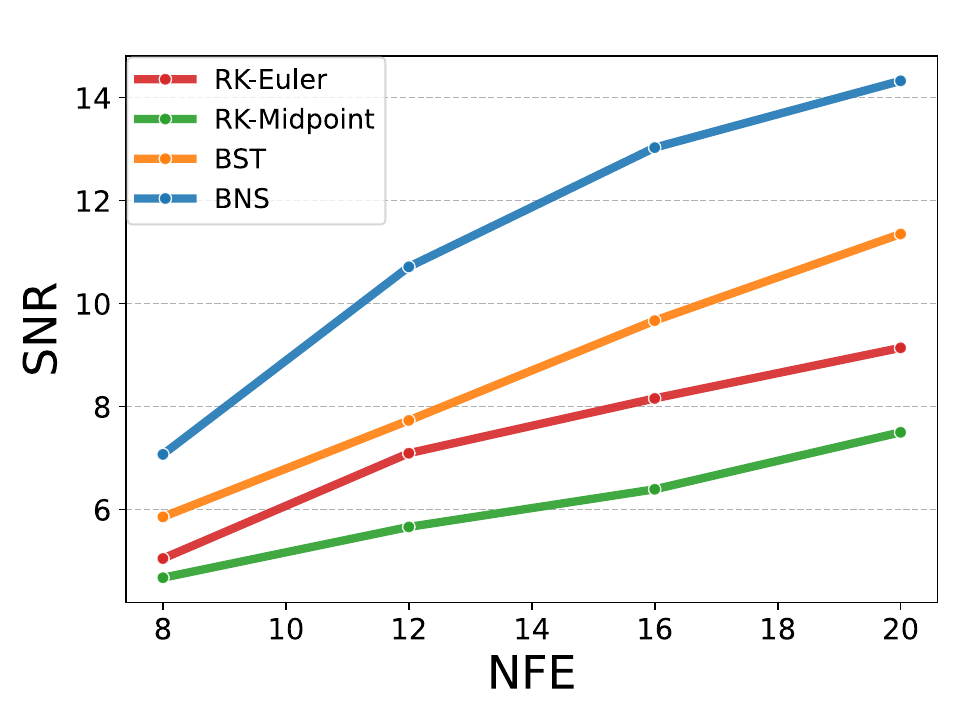} \\
        {\quad \ \scriptsize CommonVoice v13.0}\\
        \includegraphics[width=0.25 \textwidth]{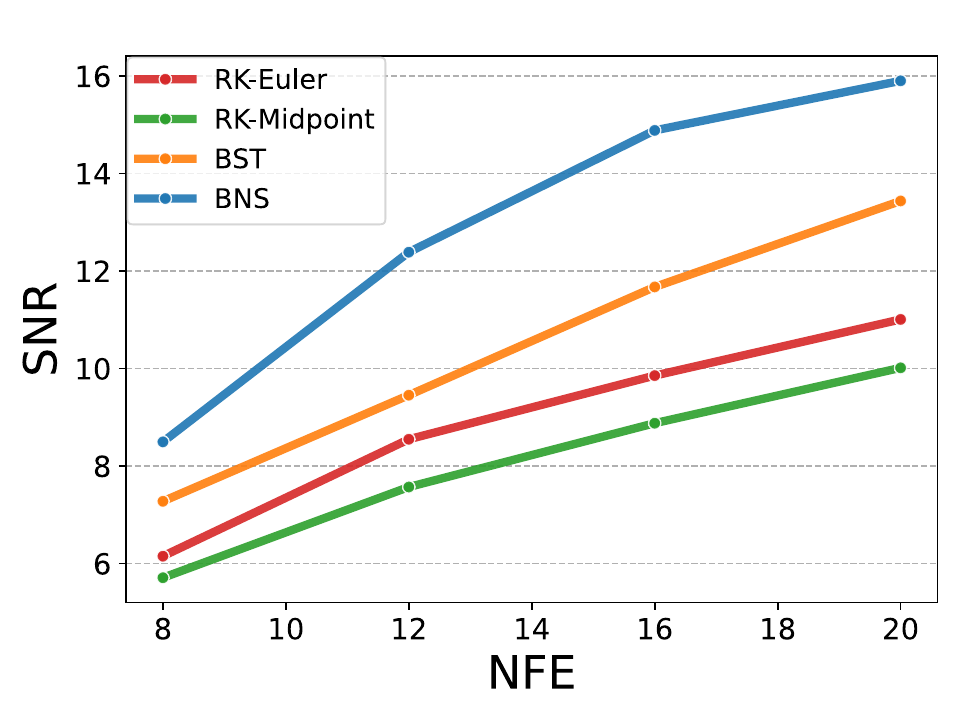}
    \end{tabular}
    &\begin{tabular}{@{\hspace{0pt}}c@{\hspace{0pt}}}
        {\quad \ \scriptsize Fisher}\\
        \includegraphics[width=0.25 \textwidth]{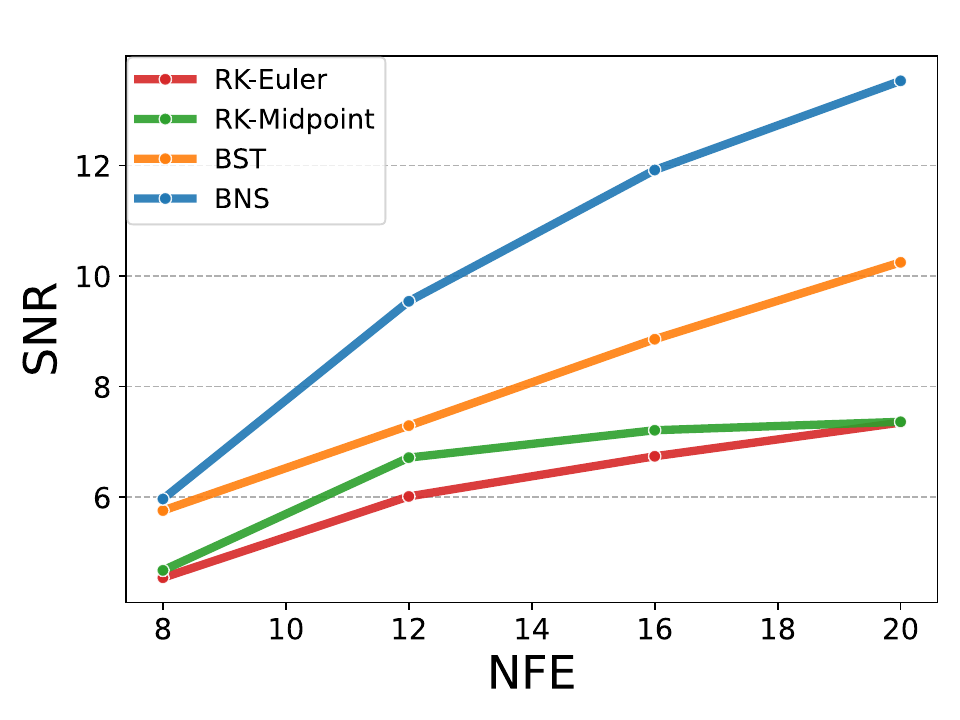} 
    \end{tabular}
    % &
    % \cincludegraphics[width=0.25 \textwidth]{figures/audio/fisher_best_val_checkpoint.pdf}
    % \includegraphics[width=0.25\textwidth]{figures/audio/fisher_best_val_checkpoint.pdf}
%
%
%
\end{tabular}
\caption{NFE vs. SNR for each solver}
\label{fig:a_graphs_t2a}
\end{figure*}

%% file: tables/audio_metrics.tex
\begin{table}
\centering
\begin{tabular}{llrrrrrrrr}
\toprule
 &  & Accent & Audiocaps & CV13 & Expresso & Fisher & LS & Spotify & Switchboard \\
Solver & NFE &  &  &  &  &  &  &  &  \\
\midrule
\multirow[t]{4}{*}{BNS} & 8 & 0.662 & 0.391 & 0.611 & 0.608 & 0.586 & 0.735 & 0.540 & 0.610 \\
 & 12 & 0.662 & 0.396 & 0.610 & 0.607 & 0.588 & 0.734 & 0.541 & 0.611 \\
 & 16 & 0.661 & 0.396 & 0.608 & 0.604 & 0.587 & 0.731 & 0.539 & 0.610 \\
 & 20 & 0.661 & 0.397 & 0.608 & 0.604 & 0.588 & 0.732 & 0.540 & 0.610 \\
\cline{1-10}
\multirow[t]{4}{*}{BST} & 8 & 0.662 & 0.387 & 0.608 & 0.605 & 0.587 & 0.732 & 0.540 & 0.610 \\
 & 12 & 0.663 & 0.398 & 0.609 & 0.605 & 0.589 & 0.733 & 0.544 & 0.611 \\
 & 16 & 0.662 & 0.399 & 0.609 & 0.602 & 0.590 & 0.731 & 0.544 & 0.611 \\
 & 20 & 0.661 & 0.397 & 0.608 & 0.602 & 0.589 & 0.731 & 0.542 & 0.611 \\
\cline{1-10}
\multirow[t]{4}{*}{Euler} & 8 & 0.662 & 0.387 & 0.609 & 0.605 & 0.584 & 0.732 & 0.544 & 0.608 \\
 & 12 & 0.664 & 0.395 & 0.611 & 0.605 & 0.588 & 0.734 & 0.546 & 0.612 \\
 & 16 & 0.665 & 0.402 & 0.612 & 0.605 & 0.590 & 0.734 & 0.548 & 0.613 \\
 & 20 & 0.665 & 0.404 & 0.612 & 0.605 & 0.591 & 0.734 & 0.550 & 0.614 \\
\cline{1-10}
\multirow[t]{4}{*}{Midpoint} & 8 & 0.664 & 0.391 & 0.608 & 0.600 & 0.592 & 0.731 & 0.543 & 0.614 \\
 & 12 & 0.664 & 0.399 & 0.608 & 0.601 & 0.593 & 0.731 & 0.547 & 0.615 \\
 & 16 & 0.662 & 0.398 & 0.608 & 0.602 & 0.590 & 0.731 & 0.543 & 0.611 \\
 & 20 & 0.661 & 0.396 & 0.608 & 0.602 & 0.589 & 0.731 & 0.539 & 0.610 \\
\cline{1-10}
RK45 & adaptive & 0.661 & 0.396 & 0.608 & 0.602 & 0.588 & 0.730 & 0.538 & 0.610 \\
\cline{1-10}
\bottomrule
\end{tabular}
\caption{Speaker similarity for each solver (higher is better)}
\label{tab:audio_ss}
\end{table}

\begin{table}
\centering
\begin{tabular}{llrrrrrrrrr}
\toprule
 &  & Accent & Audiocaps & CV13 & Expresso & Fisher & LS & LS TTS & Spotify & Switchboard \\
Solver & NFE &  &  &  &  &  &  &  &  &  \\
\midrule
\multirow[t]{4}{*}{BNS} & 8 & 0.86 & 3.98 & 3.04 & 3.11 & 7.71 & 3.01 & 3.12 & 3.48 & 11.02 \\
 & 12 & 0.98 & 3.61 & 3.38 & 3.05 & 7.66 & 3.41 & 3.23 & 2.60 & 11.39 \\
 & 16 & 1.02 & 3.69 & 3.10 & 3.09 & 7.75 & 3.16 & 3.25 & 2.59 & 12.54 \\
 & 20 & 1.07 & 3.56 & 3.07 & 3.21 & 7.87 & 3.27 & 3.33 & 2.58 & 10.50 \\
\cline{1-11}
\multirow[t]{4}{*}{BST} & 8 & 0.90 & 3.87 & 3.16 & 3.11 & 7.77 & 3.18 & 3.06 & 2.89 & 10.17 \\
 & 12 & 1.02 & 4.05 & 3.16 & 3.21 & 7.42 & 3.22 & 3.19 & 3.16 & 9.83 \\
 & 16 & 1.04 & 3.74 & 3.11 & 3.17 & 7.50 & 3.35 & 3.16 & 3.16 & 10.35 \\
 & 20 & 0.95 & 3.81 & 3.41 & 2.99 & 7.58 & 4.26 & 3.18 & 2.98 & 11.19 \\
\cline{1-11}
\multirow[t]{4}{*}{Euler} & 8 & 0.90 & 3.49 & 3.38 & 3.05 & 7.05 & 3.31 & 2.81 & 2.81 & 12.36 \\
 & 12 & 0.98 & 3.79 & 3.13 & 3.05 & 7.71 & 2.99 & 2.92 & 3.44 & 10.75 \\
 & 16 & 0.99 & 3.73 & 3.35 & 3.11 & 7.37 & 3.12 & 3.04 & 3.65 & 9.40 \\
 & 20 & 0.95 & 3.74 & 3.13 & 3.11 & 7.83 & 3.16 & 3.16 & 2.80 & 9.82 \\
\cline{1-11}
\multirow[t]{4}{*}{Midpoint} & 8 & 1.03 & 4.25 & 3.26 & 3.05 & 7.66 & 3.10 & 3.03 & 3.95 & 9.46 \\
 & 12 & 0.95 & 3.97 & 3.37 & 3.17 & 7.30 & 3.29 & 3.12 & 3.32 & 7.84 \\
 & 16 & 0.98 & 3.95 & 3.34 & 3.19 & 7.43 & 3.50 & 3.19 & 2.83 & 10.72 \\
 & 20 & 1.08 & 3.81 & 3.24 & 3.17 & 7.67 & 3.12 & 3.13 & 2.60 & 12.33 \\
\cline{1-11}
RK45 & adaptive & 1.04 & 3.76 & 3.43 & 3.13 & 7.67 & 3.27 & 3.31 & 2.88 & 10.75 \\
\cline{1-11}
\bottomrule
\end{tabular}
\caption{WER for each solver (lower is better)}
\label{tab:audio_wer}
\end{table}

%% file: tables/table_fm_training_hyper_parameters.tex
\begin{table}
\centering
% \resizebox{\textwidth}{!}{
\begin{tabular}{l c c c }
\toprule
  & CIFAR10  & ImageNet-64 & ImageNet-128  \\
  & FM-OT  & $\eps$-VP;FM-OT;FM/$v$-CS & FM-OT \\
\midrule
Channels & 128 & 196 & 256  \\
Depth  & 4  & 3  & 2  \\
Channels multiple  & 2,2,2 & 1,2,3,4 & 1,1,2,3,4  \\
Heads  & 1 & - & - \\
Heads Channels  & - & 64 & 64  \\
Attention resolution  & 16 & 32,16,8 & 32,16,8 \\
Dropout  & 0.3 & 1.0 & 0.0  \\
Effective Batch size  & 512 & 2048 & 2048  \\
GPUs  & 8 & 64 & 64  \\
Epochs  & 3000 & 1600 & 1437\\ 
Iterations  & 300k &  1M & 900k\\
Learning Rate  & 1e-4 & 1e-4 & 1e-4 \\
Learning Rate Scheduler  & constant & constant  &Poly Decay\\
Warmup Steps & - & - & 5k\\
P-Unconditional  & - & 0.2 & 0.2\\
Guidance scale  & - & 0.20 (vp,cs), 0.15 (ot) & 0.5\\
Total parameters count  & 55M & 296M & 421M\\
\bottomrule
\end{tabular}
% }
\caption{CIFAR10 and ImageNet Pre-trained models' hyper-parameters. }
\label{tab:training_hyper-params}
\end{table}